%% file: main_icml.tex
\documentclass{article}

\usepackage{hyperref}

\input{preamble}

\usepackage[accepted]{icml2024}

\usepackage{amsmath}
\usepackage{amssymb}
\usepackage{mathtools}
\usepackage{amsthm}

\theoremstyle{plain}

\theoremstyle{definition}

\theoremstyle{remark}

\usepackage[textsize=tiny]{todonotes}

\icmltitlerunning{Graph Neural Networks with a Distribution of Parametrized Graphs}

\begin{document}

\twocolumn[
\icmltitle{Graph Neural Networks with a Distribution of Parametrized Graphs}



\icmlsetsymbol{equal}{*}

\begin{icmlauthorlist}
\icmlauthor{See Hian Lee}{eee,ntu}
\icmlauthor{Feng Ji}{eee,ntu}
\icmlauthor{Kelin Xia }{spms,ntu}
\icmlauthor{Wee Peng Tay}{eee,ntu}

\end{icmlauthorlist}

\icmlaffiliation{eee}{School of Electrical and Electronic Engineering}
\icmlaffiliation{spms}{School of Physical and Mathematical Sciences}
\icmlaffiliation{ntu}{Nanyang Technological University, Singapore}

\icmlcorrespondingauthor{Feng Ji}{jifeng@ntu.edu.sg}
\icmlcorrespondingauthor{Wee Peng Tay}{wptay@ntu.edu.sg}

\icmlkeywords{Machine Learning, ICML}

\vskip 0.3in
]



\printAffiliationsAndNotice{}

\begin{abstract}
Traditionally, graph neural networks have been trained using a single observed graph. However, the observed graph represents only one possible realization. In many applications, the graph may encounter uncertainties, such as having erroneous or missing edges, as well as edge weights that provide little informative value. To address these challenges and capture additional information previously absent in the observed graph, we introduce latent variables to parameterize and generate multiple graphs. The parameters follow an unknown distribution to be estimated. We propose a formulation in terms of maximum likelihood estimation of the network parameters. Therefore, it is possible to devise an algorithm based on Expectation-Maximization (EM). Specifically, we iteratively determine the distribution of the graphs using a Markov Chain Monte Carlo (MCMC) method, incorporating the principles of PAC-Bayesian theory. Numerical experiments demonstrate improvements in performance against baseline models on node classification for both heterogeneous and homogeneous graphs.
\end{abstract}

\section{Introduction}
Graph Neural Networks (GNNs) have facilitated graph representational learning by building upon Graph Signal Processing (GSP) principles and expanding their application in the domains of machine learning. Moreover, GNNs have demonstrated their effectiveness across a wide range of tasks in domains such as chemistry \citep{chemistry2017}, recommendation systems \citep{recommender2018,recommender2022}, financial systems \citep{financialhyp2021} and e-commerce settings \citep{ecommerce}, among others. However, GSP and GNNs conventionally rely on a fixed graph shift operator, such as the adjacency or Laplacian matrix, to analyze and learn from graph data, assuming that the given graph is accurate and noise-free. This approach has inherent limitations, considering that graph data is often uncertain.

The uncertainty is due to the existence of multiple potential variations in graph constructions as a universal optimal method does not exist. Structural noise, which includes missing or spurious edges, and the absence of informative edge weights, further contributes to the uncertainty in graph data \citep{bgcn, gsnoise2023}. Handling this uncertainty is crucial as the graph directly influences the results of both GSP and GNNs \citep{gspbioreview}. 

Several GNN works have recognized that the provided graph in benchmark datasets is suboptimal. For example, in \cite{oversquashing2022}, a method was introduced to enhance the provided graph by rewiring it at graph bottlenecks. Similarly, in \cite{li2021cgnn} and \cite{ze20curvature}, approaches were developed to reweigh edges, to reduce information flow at cluster boundaries. Another perspective involves considering the given or observed graph as a particular realization of a graph model, as discussed in \cite{bgcn} (cf. \cref{sec:related_works} on related works). In their work, a Bayesian framework was adopted to learn a more robust model that can withstand perturbations in graph topology. These collective efforts underscore the common observation that the observed graph is often imperfect, and determining the optimal graph is a non-trivial task, as it depends on both the physical connections and the edge weights, which regulate the rates of information transmission  \citep{gsptsp2023}.


Our work aligns with the viewpoint presented in \cite{bgcn}. We conceptualize the observed graph as an individual instance originating from a distribution of graphs, which is influenced by one or more latent parameters. Nevertheless, in contrast to \cite{bgcn} which proposed a Bayesian framework, we propose an EM framework for graph learning and name our model EMGNN. Even though both are probabilistic frameworks, the focus is distinctly different. In the case of the Bayesian framework of \cite{bgcn}, the focus is on estimating the posterior distribution of model parameters given the data. As such, model parameters are deemed as random variables trained by a series of characteristically similar graphs. Meanwhile, in our EM framework, we seek to maximize the log-likelihood of the observed data in conjunction with the latent variables. Additionally, we permit the generated graphs to demonstrate more pronounced variations. 
Our main contributions are:
\begin{itemize}
    \item We present a general framework for modeling the distribution of graphs to handle uncertainty in graph data. The learned distribution provides valuable insights into our model's behavior.
    \item We formulate the graph learning problem as a maximum likelihood estimation (MLE) so that tools from statistical learning can be applied. The new objective subsumes the classical objective of minimizing empirical loss if the graph is deterministic. 
    \item We evaluate our model on nine datasets in two distinct applications, and observe promising performance compared to the respective baseline methods.
    \item We inspect the learned graph distribution, confirming that it effectively captures the intricacies of heterogeneous graph datasets, thus validating the utility of our model and framework.
\end{itemize}
Notations are in \cref{sec:lon} and proofs are in \cref{sec:tdi}.

\section{Preliminaries}\label{sec:prelims}

\subsection{Graph Neural Networks}
Graph neural networks \citep{gcnv2,KanZhaSon,brody2022how,9413417,zhao2023graph}, which are neural networks designed to operate on graphs, typically employ the message-passing framework. Within this framework, the features of each node are updated by integrating with those of its neighboring nodes. 

More specifically, suppose that we have a graph $G=(V, E)$, where $V$ is the set of vertices and $E$ is the set of edges. Moreover, each node $v \in V$ is associated with (initial) node features represented by $x_v^0$.  The node features can then be updated in the $k$-th layer as follows:
\begin{align} \label{eq:vsw}
    x_v^k = \sigma\big(W^k \text{AGGR}(\{x_v^{k-1} \mid v \in \mathcal{N}(v)\})\big)
\end{align}
where $\sigma$ is an activation function, $W^k$ are the learnable weights in the $k$-th layer and $\mathcal{N}(v)$ is the set of neighbors of $v$. AGGR is a message aggregation function. The choice of AGGR defines various variants of GNNs \citep{xu2018gin}. For example, the mean operator yields Graph Convolutional Networks (GCN) \citep{kipf2016semi}, while using the attention mechanism results in Graph Attention Networks (GAT) \citep{velickovic2018graph}.

In a GNN with $K$ layers, the last layer outputs features $\{x_1^K, \dots, x^K_{\mid V \mid}\}$. For node classification, these features can be used directly. Meanwhile, for a graph-level task, a READOUT graph-level pooling function is needed to obtain the graph-level representation \citep{xu2018gin}.

\subsection{Signal Processing over a Distribution of Graphs} \label{sec:spo}

GNN is closely tied to GSP's theory \citep{Shu13}. Briefly, given an undirected graph $G$, we consider a fixed graph shift operator $S$ such as its adjacency or Laplacian matrix. A graph signal is a vector $\bx = (x_v)_{v\in V}$ that associates a number $x_v$ to each node $v \in V$. Intuitively, applying the linear transformation $S$ to $\bx$ is considered as a ``shift'' of $\bx$. If $S$ is the normalized adjacency matrix, then it amounts to the AGGR step of \cref{eq:vsw} for GCN. More generally, if $P(\cdot)$ is a single variable polynomial, then plugging in $S$ results in the matrix $P(S)$, which is called a \emph{convolution filter} in GSP. This notion of convolution appears in \cite{Def16}, and has been widely used since then.  

On the signal processing side, \cite{gsptsp2023} has developed a theory that generalizes traditional GSP. The authors propose a signal processing framework assuming a family of graphs are considered simultaneously, to tackle uncertainties in graph constructions. Formally, it considers a distribution $\mu$ of graph shift operators $S_{\lambda}$ parametrized by $\lambda$ in a sample space $\Lambda$. The work develops corresponding signal processing notions such as Fourier transform, filtering, and sampling. In particular, a convolution takes the form $\E_{\lambda \sim \mu}[P_{\lambda}(S_{\lambda})]$, where $P_{\lambda}(\cdot)$ is a polynomial and $P_{\lambda}(S_{\lambda})$ is an ordinary convolution with shift $S_{\lambda}$. Our work is based on the idea that replaces $P_{\lambda}(S_{\lambda})$ with a more general filter such as a GNN model. As a preview, unlike \cite{gsptsp2023}, we introduce an EM framework that simultaneously estimates model parameters and the distribution $\mu$.

\section{The Problem Formulation}

\subsection{Distributions for Different Graph Types} \label{sec:dfd}
Here, we outline how $\Lambda$, a parameter sample space, arises for different graph types, showcasing why the proposed framework is useful for graph-related tasks. Parameters $\lambda \in \Lambda$ can be scalars, vectors, or more general forms, enabling task-specific graph parameterization and diverse graph generation. For instance, a specific $\lambda$ can indicate the edge weight for an edge type in heterogeneous graphs and be employed to create varied weighted graphs. Details are task-specific and provided in \cref{sec:exp}. Though the schemes are simple and intuitive, there may be alternatives for $\Lambda$ based on other factors.


\emph{Heterogeneous graphs} are graph structures characterized by the presence of multiple node types and multiple edge types, imparting a greater degree of complexity compared to \emph{homogeneous graphs}, which consist of a single node and edge type. For a heterogeneous graph, the insight is that we assign a parameter to each edge type, whose distribution is to be estimated and used. Intuitively, in a model based on message passing, the parameters for different edge types are interpreted as different information transmission rates. Such information is not observed in the given graph. 

It is less obvious how $\Lambda$ can be constructed for a homogenous graph. Intuitively, we assume that the observed graph contains ``noisy'' edges and has missing edges (cf.\ \citet{bgcn}). Therefore, parameters that are interpreted as probabilities for adding and removing edges from the observed graph, can be introduced (see an example in \figref{fig:fah}). 

\begin{figure}
    \centering
    \includegraphics[scale=0.6]{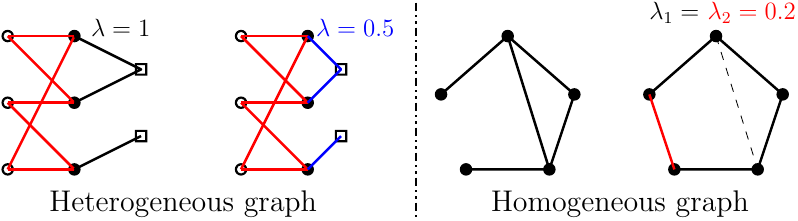}
    \caption{For a heterogenous graph, we may use a parameter $\lambda$ to control the information transmission rate for each edge type. For example, choosing $\lambda=1$ or $\lambda=0.5$ for the edge type between ``disc'' and ``square'' nodes yields different weighted graphs. Conversely, in a homogeneous example with 5 initial edges, by choosing $\lambda_1 = \lambda_2 = 0.2$, $20\%$ of the initial and missing edges are randomly removed and added, potentially forming a ``pentagon''.}
    \label{fig:fah}
\end{figure}

Though we have different setups, a unified framework to deal with both is proposed in the next section. To emphasize the importance of having the correct $\mu$ on $\Lambda$, we note the following result. Details are in \cref{sec:dat} to avoid overloading the discussion with terms not used in the sequel.

\begin{Theorem} \label{thm:iit}
    (Informal) If the parameterization $\lambda \in \Lambda \mapsto S_{\lambda}$ is sufficiently continuous, then the $\mu$-expected feature representation of a GNN model, whose layers are of the form (\ref{eq:vsw}), changes continuously as the distribution $\mu$ varies.
\end{Theorem}

 {\bf \emph{Discussion}}: Intuitively, the result claims that if the parameterization is sufficiently regular, then similar distributions on $\Lambda$ yield GNN models with similar feature representations. Hence, suppose there is a true graph $G_0$ parameterized $\lambda_0$, whose associated GNN model gives a good feature representation. Instead of using a possibly ``noisy'' observed graph $G$ parameterized by $\lambda$, it might be beneficial to use a distribution $\mu$ on $\Lambda$ closer than $\delta_{\lambda}$ to $\delta_{\lambda_0}$, even if $\mu$ is not a delta distribution. It is less restrictive to allow non-delta graph distributions. Moreover, we are also motivated by the insight that there might be several node connections that contribute with different importance to a learning task.

\subsection{Maximum Likelihood Estimation}\label{subsec:formulation}

Motivated by the previous subsection, we consider a distribution $\mu$ on a parameter (sample) space $\Lambda \subset \mathbb{R}^r$ of graphs $\{G_{\lambda}, \lambda\in \Lambda\}$, with a fixed set of nodes $V$. The space $\Lambda$ can be finite, countably infinite, or even uncountable. For each $G_{\lambda}$, there is a corresponding shift operator $S_{\lambda}$. We usually assume that $\mu$ has a density function $p(\cdot)$ w.r.t.\ a base measure on $\Lambda$. For example, if $\Lambda$ is finite, we can use the discrete counting measure as the base measure. On the other hand, if $\Lambda$ is a compact interval in $\mathbb{R}$, then we can choose the Lebesgue measure as the base measure. 

Assume that each node $v \in V$ is associated with features $x_v$. They are collectively denoted by $\bx$. 
Our framework depends on a \emph{fixed GNN model architecture $\Psi$}, e.g., GCN. It outputs the learned embeddings $\bz = \Psi(\lambda,\bx;\btheta)$ given the node features $\bx$, the graph parameter $\lambda$, and the GNN model parameters $\btheta$. These in turn are used to determine a vector of labels $\hat{\by}$. For a task-specific loss $\ell(\cdot,\cdot)$ that compares predicted $\hat{\by}$ and true label (vector) $\by$, we may compute ${L}_{\bX}(\lambda,\btheta) = \ell(\hat{\by},\by)$. We use $\bX$ to denote the full information $\{\bx,\by\}$. We interpret $\bX$ as a sample from a random variable, denoted by $\mathfrak{X}$, of collective information of features and labels.

An example of $\Psi$ is the model described by \cref{eq:vsw}. We may also allow parameters $\btheta$ and $\lambda$ to determine $W^k$. For example, if $\btheta = \{\theta^k_1,\theta^k_2 \mid 1\leq k\leq K\}$ and $\lambda$ is a scalar parameter, then one can choose a linear combination $W^k=\theta^k_1 + \lambda\theta^k_2$. Moreover, AGGR is determined by the shift $S_{\lambda}$ associated with $G_{\lambda}$. 

In general, as $\lambda$ follows an unknown distribution $\mu$, it is hard to find the optimal $\btheta$ by minimizing $\E_{\lambda\sim \mu}[{L}_{\bX}(\lambda,\btheta)]$ directly. On the other hand, the EM algorithm \citep{Bis06} enables the joint estimation of $\mu$ and $\btheta$ if we can reformulate the objective as an MLE.

To minimize the loss given $\bX$, the parameter $\btheta$ is determined by $\lambda$ and vice versa. Therefore, $\Psi(\cdot,\bx;\cdot)$ becomes a random GNN model that depends on $\lambda, \btheta$ and input $\bx$. We aim to \emph{identify a realization of the random models that makes the observation $\bX$ likely}, i.e., there is less discrepancy between the estimator labels $\hat{\by}$ and ground truth labels $\by$ measured by the loss $\ell(\hat{\by},\by)$. Motivated by the discussions above, we consider the likelihood function $p(\lambda,\bX \mid \btheta)$ on $\btheta$ and formulate the following MLE as our objective:
\begin{align} \label{eq:tap}
\btheta^{*} = \argmax_{\btheta} p(\bX \mid 
\btheta) = \argmax_{\btheta}\E_{\lambda\sim \mu} [p(\lambda,\bX \mid \btheta)]. 
\end{align}

Before discussing the main algorithm in subsequent subsections, we preview the roles of $\mu$ and ${L}_{\bX}(\cdot,\cdot)$ in the algorithm. We shall see that the EM algorithm outputs a distribution $\hat{\mu}$ of $\lambda$, serving as an estimate of $\mu$, by leveraging the PAC-Bayesian framework \citep{Gue19}. In this framework, the density of $\lambda$ is proportional to the Gibbs posterior, depending on $\ell(\cdot,\cdot)$. Consequently, $\hat{\mu}$ assigns higher probability density to $\lambda$ when the loss ${L}_{\bX}(\lambda,\btheta^*)$ is lower. Therefore, we are minimizing the given loss as a main component of the algorithm. In the above formulation, $\btheta$ is the parameter and $\mu$ is unknown before attaining $\btheta$. However, once $\btheta$ is determined, an estimation of $\mu$ is obtained using the Gibbs posterior. 

\begin{Example}
Assume that $\mu$ is the delta distribution $\delta_{\lambda_0}$ supported on $\lambda_0$, so that the graph $G_{\lambda_0}$ is deterministic. If we consider the Gibbs posterior, then we have $p(\bX \mid \btheta) \propto \exp(-\eta \ell(\hat{\by},\by))$ for a hyperparameter $\eta$, where $\hat{\by}$ depends on both $\bX=\{\bx,\by\}$ and $\btheta$. Thus, maximizing $p(\bX\mid \btheta)$ is equivalent to the classical objective of minimizing $\ell(\hat{\by},\by)$. 
\end{Example}

\section{The Proposed Method and Its Derivation}\label{sec:method}

\subsection{EM for GNN}\label{subsec:EM}
Optimizing \cref{eq:tap} directly can be challenging, and we utilize the EM algorithm that employs an iterative approach alternating between the E-step and the M-step. Adapted to our setting, the process unfolds as follows:
\begin{enumerate}[(a)]
    \item E-step: Given parameters $\btheta^{(t)}$ at the $t$-th iteration, we compute the expectation as the Q-function:
    \begin{align} \label{eq:qbm}
        Q(\btheta \mid \btheta^{(t)}) = \E_{\lambda\sim p(\cdot \mid \bX,\btheta^{(t)})}[\log p(\lambda,\bX\mid \btheta)]. 
    \end{align}
    \item M-step: $\btheta^{(t+1)}$ is updated as $\argmax_{\btheta} Q(\btheta\mid \btheta^{(t)})$. 
\end{enumerate}   

{\bf \emph{The E-step}}: In the $t$-th iteration, in the same spirit as the PAC-Bayesian framework \citep{Gue19}, we apply the Gibbs posterior and assume that  
\begin{align}
p(\lambda,\bX \mid \btheta) & \propto  \exp(-\eta^{(t)} {L}_{\bX}(\lambda,\btheta))\pi_0(\lambda,\bX), \label{p-PAC}
\end{align}
for a tunable hyperparameter $\eta^{(t)}$, while $\pi_0(\cdot)$ is a prior density of the joint $(\lambda,\bX)$ independent of $\btheta$, representing our initial knowledge regarding $\lambda$ and $\bX$. In this expression, ${L}_{\bX}(\lambda,\btheta)$ implicitly depends on the observations $\bX$. 
The normalization constant is given by 
\begin{equation}\label{eq:Ctheta}
\resizebox{\linewidth}{!}{$ \begin{aligned}
    C(\btheta) & = \int_{(\lambda,\bX') \in \Lambda\times \mathfrak{X}}\exp(-\eta^{(t)} {L}_{\bX'}(\lambda,\btheta))\pi_0(\lambda,\bX') \ud (\lambda,\bX') \\
    & = \E_{(\lambda,\bX') \sim \pi_0}[ \exp(-\eta^{(t)} {L}_{\bX'}(\lambda,\btheta)) ].
\end{aligned}$}
\end{equation}


As a prior belief, we treat the observed $\bX$ as a typical sample such that the above average is (approximately) the same as the average over graphs by fixing $\bX$ (cf.\ \cref{sec:tio}).  
We assume that for each fixed $\bX$, there exists some prior distribution with density $p_{0,\bX}(\cdot)$ on $\Lambda$ such that:
\begin{align} \label{eq:elb}
\begin{aligned}
& \E_{(\lambda,\bX') \sim \pi_0}[ \exp(-\eta^{(t)} {L}_{\bX'}(\lambda,\btheta)) ] \\
\approx& \int_{\lambda \in \Lambda}\exp(-\eta^{(t)} {L}_{\bX}(\lambda,\btheta))p_{0,\bX}(\lambda) \ud (\lambda) \\
=& \E_{\lambda \sim p_{0,\bX}}[\exp(-\eta^{(t)} {L}_{\bX}(\lambda,\btheta)) \given \bX].
\end{aligned}
\end{align}
\emph{For notational simplicity, we use $p_0(\cdot)$ to denote $p_{0,\bX}(\cdot)$. Correspondingly, $\E_{\lambda \sim p_{0}}[\exp(-\eta^{(t)} {L}_{\bX}(\lambda,\btheta)) \given \bX]$ is denoted by $\E_{\lambda \sim p_{0}}[\exp(-\eta^{(t)} {L}_{\bX}(\lambda,\btheta))]$}. Hence, we have 
\begin{align} \label{eq:cml}
    C(\btheta) &= \E_{\lambda \sim p_{0}}[\exp(-\eta^{(t)} {L}_{\bX}(\lambda,\btheta))].
\end{align}
On the other hand, given $\btheta^{(t)}$, we estimate $p(\lambda \mid \bX, \btheta^{(t)})$ in the subscript of $\mathbb{E}$ in (\ref{eq:qbm}). From \cref{p-PAC}, we have
\begin{align*}
p(\lambda \mid \bX, \btheta^{(t)}) 
&= \frac{p(\lambda, \bX \mid \btheta^{(t)})}{p(\bX \mid \btheta\tc{t})}\\
&\propto \exp(-\eta\tc{t} {L}_{\bX}(\lambda,\btheta\tc{t}))\frac{\pi_0(\lambda,\bX)}{p(\bX \mid \btheta\tc{t})}.
\end{align*}%
We assume that there is a prior $p'_{0,t}(\cdot)$ such that $p'_{0,t}(\lambda)\propto \frac{\pi_0(\lambda,\bX)}{p(\bX \mid \btheta\tc{t})}$, which is independent of $\btheta$. However, it is a function of $t$ as $\btheta^{(t)}$ depends on $t$. By fixing $\bX$, the posterior is written as 
\begin{align} \label{eq:plm}
p(\lambda \mid \bX, \btheta^{(t)}) \propto \exp(-\eta^{(t)} {L}_{\bX}(\lambda,\btheta^{(t)}))p'_{0,t}(\lambda).\end{align}
In our framework, we do not need to estimate the normalization constant for \cref{eq:plm}.  


\begin{Remark} \label{rmk:fta}
From the above discussion, we see that priors $p_0(\cdot)$ and $p_{0,t}'(\cdot)$ play important roles. We discuss their choices in \cref{sec:exp} below. However, it is always desirable to have a weaker prior assumption, under which the optimizer can still be readily estimated. 
\end{Remark}

{\bf \emph{The M-step}}: We now analyze the $Q$-function in more detail. \emph{For convenience, we use $p_t(\lambda)$ to denote $p(\lambda\mid \bX,\btheta^{(t)})$.} 

Combining \cref{eq:cml,eq:plm}, we express $Q(\btheta \mid \btheta^{(t)})$ in (\ref{eq:qbm}) as: 
\begin{align*}
& Q(\btheta \mid \btheta^{(t)}) \\
= & \E_{\lambda\sim p_t}[\log \frac{\exp(-\eta^{(t)} {L}_{\bX}(\lambda,\btheta))\pi_0(\lambda,\bX)}{C(\btheta)}] 
\\
= & -\eta^{(t)}\E_{\lambda\sim p_t}[{L}_{\bX}(\lambda,\btheta)] + D - \log C(\btheta),
\end{align*}
where $D$ is a constant independent of $\btheta$. 

To estimate $\log C(\btheta)$, consider the Jensen inequality: 
\begin{align*}
-\eta^{(t)}\E_{\lambda\sim p_0}[{L}_{\bX}(\lambda,\btheta)]  \leq \log C(\btheta).
\end{align*}
This means that if $\log C(\btheta)$ is small, then necessarily so is $-\eta^{(t)}\E_{\lambda\sim p_0}[{L}_{\bX}(\lambda,\btheta)]$. On the other hand, \cite{Teh06} proposes to use $\E[\log Y] + \dfrac{\Var(Y)}{2\mathbb{E}(Y)^2}$ to approximate $\log \E[Y]$ for a random variable $Y$. This is derived from the second-order Taylor expansion of $\log Y$ at $\log \E[Y]$. In our case, we have
\begin{align} \label{eq:lca}
\begin{split}
    \log C(\btheta) & \approx -\eta^{(t)}\E_{\lambda\sim p_0}[{L}_{\bX}(\lambda,\btheta)] \\
         + & \frac{\Var(\exp(-\eta^{(t)}{L}_{\bX}(\lambda,\btheta)))}{2\parens*{\E_{\lambda\sim p_0}[\exp(-\eta^{(t)}{L}_{\bX}(\lambda,\btheta))]}^2}.  
\end{split}
\end{align}
If $-\eta^{(t)}\E_{\lambda\sim p_0}[{L}_{\bX}(\lambda,\btheta)]$ is the dominant component, then we may use $-\eta^{(t)}\E_{\lambda\sim p_0}[{L}_{\bX}(\lambda,\btheta)]$ as a proxy for $\log C(\btheta)$, which is more manageable. In \cref{subsubsec:numverify}, we numerically verify that this is indeed the case for our applications.

Hence, $Q(\btheta \mid \btheta^{(t)})$ is approximated by 
\begin{align*}
-\eta^{(t)}\Big(\E_{\lambda\sim p_t}[{L}_{\bX}(\lambda,\btheta)]- \E_{\lambda\sim p_0}[{L}_{\bX}(\lambda,\btheta)]\Big) + D.
\end{align*}
\emph{In summary}, if we disregard $\eta^{(t)}$ and $D$, which are independent of $\btheta$, we may minimize the following in the M-step: 
\begin{align}\label{eq:wqt}
\begin{split} 
    J(\btheta) & = \E_{\lambda\sim p_t}[{L}_{\bX}(\lambda,\btheta)]- \E_{\lambda\sim p_0}[{L}_{\bX}(\lambda,\btheta)]\\
    & = \int_{\lambda \in \Lambda} \big(p_t(\lambda)-p_0(\lambda)\big) {L}_{\bX}(\lambda,\btheta) \ud \lambda.
\end{split}
\end{align}

\subsection{The Proposed Algorithm: EMGNN}

To minimize $J(\btheta)$ in \cref{eq:wqt}, our strategy is to re-express it as an expectation. For this purpose, we introduce a proposal distribution. Let $q(\cdot)$ be the density function of a probability distribution on the sample space $\Lambda$ whose support includes that of $p_0$. Then we have:
\begin{align*}
    J(\btheta) & = \int_{\lambda \in \Lambda}q(\lambda) \frac{p_t(\lambda)-p_0(\lambda)}{q(\lambda)} {L}_{\bX}(\lambda,\btheta) \ud \lambda \\ & = 
    \E_{\lambda \sim q}[\frac{p_t(\lambda)-p_0(\lambda)}{q(\lambda)} {L}_{\bX}(\lambda,\btheta)].  
\end{align*}
We propose to minimize $J(\btheta)$ by first randomly drawing samples $\Lambda_{T'} = \{\lambda_1,\ldots,\lambda_{T'}\}$ according to the density $q(\cdot)$. Following that, we successively apply gradient descent to $\frac{p_t(\lambda_{t'})-p_0(\lambda_{t'})}{q(\lambda)} {L}_{\bX}(\lambda_{t'},\btheta)$ to update $\btheta$. Finally, given \cref{eq:plm}, $p_t(\lambda)$ can be approximated by an empirical distribution if we apply an MCMC method. The overall algorithm is summarized in \cref{alg:cap2} and illustrated in \cref{fig:emgnn}.

\begin{figure}
    \centering
    \includegraphics[width=0.98\columnwidth]{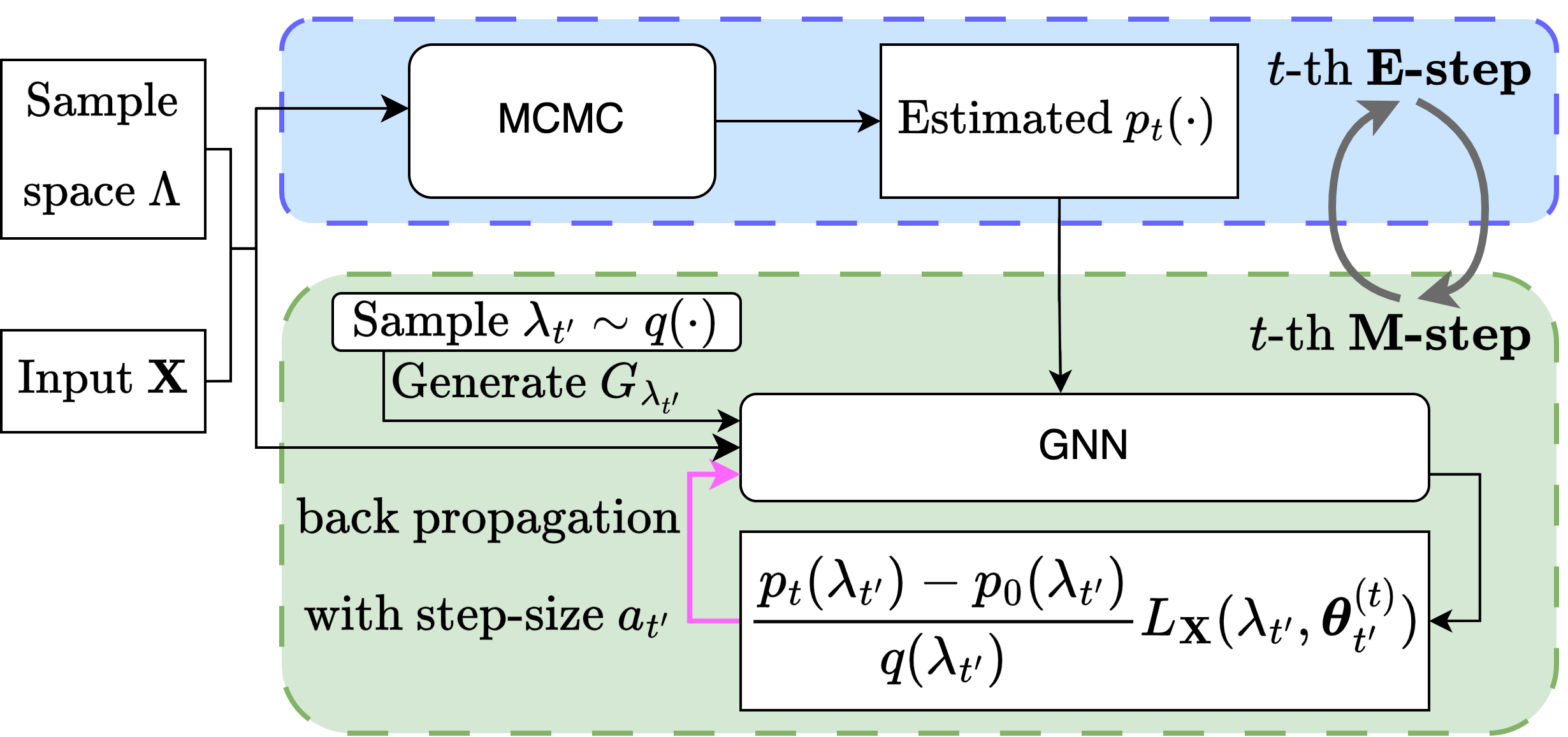}
    \caption{Illustration of EMGNN.}
    \label{fig:emgnn}
\end{figure}

\begin{algorithm}[!htb]
\caption{EMGNN}\label{alg:cap2}
\textbf{Input}: The observed graph $G$,\\
    The node features $\bx$,\\
    The number of EM iterations $T$, \\
    The number of epochs per M-step $T'$,\\
    The sample space $\Lambda$,\\
    Prior distributions $p_0, p'_{0,t}$ and $ q$,\\
    Function to convert samples to empirical distribution $g(\cdot)$,\\
    Task-specific function to generate $\lambda$ influenced graphs $h(\cdot)$,\\
    A non-increasing step-size $a_{t'}$.\\
\textbf{Output}: The learned representation $\bz$.\\
\textbf{Initialization}: Warm up $\Psi$ using $G$. 
\begin{algorithmic}[1]

\FOR{$t=1$ {\bfseries to} $T$}
  \STATE $\text{AcceptedList}^{(t)}$ $\gets$ MCMC($\Lambda$, $p'_{0,t}$) 
  \STATE $\text{EmpProbDict}^{(t)} \gets g(\text{AcceptedList}^{(t)})$ 
  \FOR{$t'=1$ {\bfseries to} $T'$} 
    \STATE Sample $\lambda_{t'} \sim q$. 
    \STATE $G_{\lambda_{t'}} \gets h(\lambda_{t'}, G)$ 

    \STATE Update via gradient descent: $\btheta_{{t'}+1}^{(t)} = \btheta_{t'}^{(t)} - a_{t'}\nabla(\frac{p_t(\lambda_{t'}) - p_0(\lambda_{t'})}{q(\lambda_{t'})}{L}_{\bX}(\lambda_{t'},\btheta_{t'}^{(t)})) $
    
    \STATE $\btheta_{{t'}+1}^{(t)}, \bz_{t'} \gets \Psi(\lambda_{t'}, \bx; \btheta_{t'}^{(t)})$


  \ENDFOR
  \STATE $\btheta^{(t+1)} \gets \btheta^{(t)}_{T'}$
  \STATE $t \gets t+1$
  
\ENDFOR

\STATE $\bz_{\text{final}} = \E_{\lambda \sim p_T(\cdot)}[\Psi(\lambda, \bx; \btheta^{(T)}_{T'})]$
\end{algorithmic}
\end{algorithm}


\begin{Remark} In practice, the choices of the prior distributions $p_{0}(\cdot), q(\cdot)$ and $ p'_{0,t}(\cdot)$ are hyperparameters. Moreover, in our experiments, $p'_{0,t}(\cdot)$ is set to be the same for every $t$. We also discretize the continuous sample space $\Lambda$ for simplicity in analysis and computation.
\end{Remark}

\begin{Remark}
If we algorithmically plug in the delta distribution supported on $\lambda_0$ and $p_0(\lambda_0) = 0$ for  $p'_{0,t}(\cdot)$ and $q(\cdot)$ respectively, then EMGNN reduces to the ordinary GNN model on the graph $G_{\lambda_0}$.
\end{Remark}




\begin{Remark}\label{rem:badlambda} Note that for the coefficient $\frac{p_t(\lambda)-p_0(\lambda)}{q(\lambda)}$, if $p_t(\lambda) < p_0(\lambda)$, then the loss ${L}_{\bX}(\lambda,\btheta)$ is to be made larger. Intuitively, in this case, a ``bad'' $\lambda$ is chosen. For the choice of $q(\cdot)$, in practice, we propose two options in \cref{sec:exp}: either the uniform distribution or $q(\cdot)=p_t(\cdot)$. Nonetheless, $q(\cdot)$ can also be other appropriate density functions. 
\end{Remark}

As we do not minimize $J(\btheta)$ directly, we justify the proposed approach under additional assumptions. We theoretically analyze the performance of the proposed (randomized) algorithm in lines 5-12 of \cref{alg:cap2}, denoted by $\mathcal{A}$. With samples $\Lambda_{T'}$, the algorithm $\mathcal{A}$ outputs $\widehat{\btheta} = \mathcal{A}(\Lambda_{T'})$. The following expression is considered in algorithm $\mathcal{A}$:
\begin{align*}
    J_{\Lambda_{T'}}(\widehat{\btheta}) = \frac{1}{T'}\sum_{\lambda_{t'} \in \Lambda_{T'}} \frac{p_t(\lambda_{t'})-p_0(\lambda_{t'})}{q(\lambda_{t'})}{L}_{\bX}(\lambda_{t'},\widehat{\btheta}).
\end{align*}
We assume that after translation and scaling by positive constant of ${L}_{\bX}(\lambda,\cdot)$ if necessary, the expression $\frac{p_t(\lambda)-p_0(\lambda)}{q(\lambda)}{L}_{\bX}(\lambda,\btheta)$ always belong to $[0,1]$. The following notions are well-known.

\begin{Definition}
A differential function $f$ is $\alpha$-Lipschitz if for all $x$ in the domain of $f$, we have $\norm{\nabla f(x)} \leq \alpha$. It is $\beta$-smooth if its gradient is $\beta$-Lipschitz.     
\end{Definition}

Denote by $\indicator{p_t(\cdot)\geq p_0(\cdot)}(\lambda)$ the indicator that is $1$ if $p_t(\lambda)\geq p_0(\lambda)$, and $0$ otherwise. Let $b_1 = \E_{\lambda\sim q}[\indicator{p_t(\cdot)\geq p_0(\cdot)}]$. Intuitively, it computes the measure of $\lambda$, for which $p_t(\cdot)$ is larger. On the other hand, let $b_2 = \E_{\lambda\sim q}[\frac{|p_t(\lambda)-p_0(\lambda)|}{q(\lambda)}]$, and $\gamma = \sup_{\lambda\in \Lambda} 1/q(\lambda)$. 

\begin{Theorem} \label{thm:afa}
Assume for any $\lambda$, the loss ${L}_{\bX}(\lambda,\cdot)$ is convex, $\alpha$-Lipschitz and $\beta$-smooth. Let $b_1,b_2,\gamma$ be defined as above. If for every $t'\leq T'$, the non-increasing step-size in the algorithm $\mathcal{A}$ satisfies $a_{t'} \leq \min\{2/(\beta\gamma), c/t'\}$ for a constant $c$, then there is a constant $C$ independent of $T',\alpha$ such that 
    \begin{align*}
        \abs*{\E_{\mathcal{A},\Lambda_{T'}}[J_{\Lambda_{T'}}(\widehat{\btheta})- J(\widehat{\btheta})]} \leq \epsilon = C \parens*{\frac{b_2^2\alpha^2}{T'}}^{\frac{1}{\beta\gamma c(1-b_1)+1}}.
    \end{align*}
\end{Theorem}

\begin{Remark} From the result, we see that if $b_1$ is close to $1$, i.e, the set $\{\lambda \mid p_t(\lambda) \geq p_0(\lambda)\}$ has a large measure, then the expected error decays at a rate close to $T'^{-1}$. 
\end{Remark}

\subsection{A Brief Discussion on Testing}

As our framework deals with a distribution of graphs, during testing, the final learned representation is
$\bz_{\text{final}} = \E_{\lambda \sim p_T}[\Psi(\lambda, \bx;\btheta^{(T)}_{T'})]$ (cf.\ \cref{thm:iit}). The learned model parameters are a particular realization of the possible random models that align with the observed data $\bX$ and the multiple graphs influence the final embeddings based on their respective likelihoods. The embedding $\bz_{\text{final}}$ is subjected to a softmax operation to obtain $\hat{\by}$ for node classification tasks, while a READOUT function is applied for graph-level tasks.

\section{EXPERIMENTS} \label{sec:exp}
We study node classification for heterogeneous and homogeneous graphs. In \cref{subsec:chemdataset}, \cref{sec:dod} and \cref{sec:com}, we study chemical datasets, provide dataset and implementation details, and discuss model complexity.  

\subsection{Heterogeneous Graphs}\label{subsec:heterograph}

\subsubsection{The Experimental Setup and Baselines} \label{subsubsec:baselinehetero}

As outlined in \cref{sec:dfd}, it is more natural to apply the framework to heterogeneous graphs. Let $G$ be such a graph and assume with $\omega$ edge types. To apply the framework, it suffices to specify the sample parameter space. We introduce a vector of latent parameters $\bm{\lambda}=\{\lambda_1, \dots, \lambda_{\omega}\}$, where each $\lambda_i\in [0,1]$ and $\sum_{i=1}^{\omega}\lambda_i =1$. The number $\lambda_i$ is the weight for the $i$-th edge type. Hence, the sample parameter space is the $(w-1)$-simplex, denoted by $\Lambda$. For a chosen $\blambda$,
the associated weighted graph $G_{\bm{\lambda}}$ has adjacency matrix 
\begin{equation}\label{eqn:Alambda}
    A_{\bm{\lambda}} = \sum_{i=1}^{\omega}\lambda_i A_i,
\end{equation} where $A_i$ is the edge type specific adjacency matrix corresponding to the $i$-th edge type. Note that when any $\lambda_i=0$, the edges of the associated edge types are removed. For our model, we discretize the interval $[0,1]$ in increments of $0.05$, and the resulting discretized sample parameter space is denoted by $\widehat{\Lambda}$.

The heterogeneous graph datasets used are the same as those in \citet{yun2019gtn} and \citet{Lee2022SGATSG}. These datasets include two citation networks DBLP and ACM, as well as a movie dataset IMDB. They have similar edge type structures (cf.\ \cref{sec:hed}). For example, IMDB has three node types (Movie (M), Actor (A), Director (D)) and two edge types (MD, MA). Hence, the parameter is $\blambda = (\lambda, 1-\lambda)$.

We assess our approach against five baseline models. Specifically, GAT and GCN are designed for homogeneous graphs, while GTN \citep{yun2019gtn}, Simple-HGN \citep{hgb} and SeHGNN \citep{Yang2023SimpleAE} are state-of-the-art models developed for heterogeneous graphs. Our framework, applicable to diverse GNNs, is exemplified with a GCN backbone, named EM-GCN. We consider different variants of EM-GCN, based on choices of $p_0(\cdot), p'_{0,t}(\cdot), q(\cdot)$, as summarized in \cref{tab:ppq}. 
\begin{table}[!htb]
    \centering
    \caption{Variants of EM-GCN with different $p_0(\cdot), p'_{0,t}(\cdot), q(\cdot)$.}
    \label{tab:ppq}
    \resizebox{0.75\columnwidth}{!}{
    \begin{tabular}{ccc|c}
    \toprule
        $p_0(\cdot)$ & $p'_{0,t}(\cdot)$ & $q(\cdot)$ & Model  \\
        \midrule
        Unif($\widehat{\Lambda}$) & Unif($\widehat{\Lambda}$) & $p_t(\cdot)$ & EM-GCN[PT] \\
        Unif($\widehat{\Lambda}$) & Unif($\widehat{\Lambda}$) & Unif($\widehat{\Lambda}$) & EM-GCN[PO] \\
        $\delta_{\lambda_0}$ & Unif($\widehat{\Lambda}$) & $p_t(\cdot)$ & EM-GCN[PD] \\
        $\delta_{\lambda_0}$ & Unif($\widehat{\Lambda}$) & Unif($\widehat{\Lambda}$) & EM-GCN[PH] \\
        
        \bottomrule
    \end{tabular}}
\end{table}

\subsubsection{Results}

\begin{table*}[!htb]
\centering
    \caption{Heterogeneous node classification task. The results shown are averaged over ten runs and accompanied by the standard deviation. The best performance is boldfaced and the second-best performance is underlined.}
    \label{results:heterogeneous}
\resizebox{0.8\textwidth}{!}{
\begin{tabular}{@{}ccccccc@{}}
\toprule
                 & \multicolumn{2}{c}{IMDB}          & \multicolumn{2}{c}{ACM}                           & \multicolumn{2}{c}{DBLP}                          \\ 
                 & Micro-F1               & Macro-F1               & Micro-F1               & Macro-F1               & Micro-F1               & Macro-F1               \\ \midrule
GCN              & 61.91 $\pm$ 0.67          & 60.91 $\pm$ 0.57          & 91.92 $\pm$ 0.40          & 92.00 $\pm$ 0.41          & 94.60 $\pm$ 0.31          & 93.88 $\pm$ 0.36          \\
GAT              & \udcloser{63.54 $\pm$ 1.10}    & 61.87 $\pm$ 0.95    &  92.61 $\pm$ 0.36    & \udcloser{92.68 $\pm$ 0.36} & 94.48 $\pm$ 0.22          & 93.74 $\pm$ 0.27          \\
GTN              & 60.58 $\pm$ 2.10   & 59.12 $\pm$ 1.58 & 92.12 $\pm$ 0.62 & 92.23 $\pm$ 0.60 & 94.17 $\pm$ 0.26 & 93.59 $\pm$ 0.40 \\
Simple-HGN  & 58.91 $\pm$ 1.06          & 58.30 $\pm$ 0.34          & \textbf{92.73 $\pm$ 0.21} & 92.56 $\pm$ 0.42   & 94.48 $\pm$ 0.38          & 93.69 $\pm$ 0.32          \\
SeHGNN  & 62.13 $\pm$ 2.38         & 60.62 $\pm$ 1.95         & 92.45 $\pm$ 0.17          & 92.51 $\pm$ 0.16          &  94.86 $\pm$ 0.14    & 94.14 $\pm$ 0.19   \\ \midrule
EM-GCN[PT]       & \textbf{64.78 $\pm$ 1.24} & \textbf{63.36 $\pm$ 0.80} & \udcloser{92.70 $\pm$ 0.26}        & \textbf{92.78 $\pm$ 0.26}        & \textbf{95.06 $\pm$ 0.39} & \textbf{94.41 $\pm$ 0.45} \\ 
EM-GCN[PO] & 63.35 $\pm$ 0.79	& \udcloser{62.25 $\pm$ 0.59}	& 92.35 $\pm$ 0.38	& 92.45 $\pm$ 0.38	& 94.95 $\pm$ 0.24	& 94.28 $\pm$ 0.28 \\
EM-GCN[PD] & 62.49 $\pm$ 0.87 	& 61.55 $\pm$ 0.71 	& 92.31 $\pm$ 0.43 & 92.41 $\pm$ 0.42 &   94.89 $\pm$ 0.17	&   94.15 $\pm$ 0.23  \\
EM-GCN[PH] & 62.01 $\pm$ 0.55 	& 61.15 $\pm$ 0.46 	& 92.18 $\pm$ 0.52 & 92.29 $\pm$ 0.52 &  \udcloser{ 95.02 $\pm$ 0.19} 	&  \udcloser{ 94.34 $\pm$ 0.20 } \\
\bottomrule
\end{tabular}}
\end{table*}

Results are shown in \cref{results:heterogeneous}. Similar to recent findings \citep{hgb}, GCN and GAT are observed to perform competitively against models designed for heterogeneous graphs such as GTN under appropriate settings. Meanwhile, EM-GCN[PT] consistently outperforms other variants in our framework, in both micro and macro F1 scores. In particular, the superior performance of EM-GCN[PT], EM-GCN[PO], EM-GCN[PD], and EM-GCN[PH] compared to GCN indicates the effectiveness of our distribution-based framework. 

EM-GCN[PT] also often surpasses baselines models with attention mechanisms, namely GAT, Simple-HGN, SeHGNN, and GTN, despite not incorporating any attention mechanisms. This could be attributed to the construction of multiple graphs, which may form instances whose information is similar to what is achieved with semantic attention. In addition, the model may extract additional useful interactions from other graph instances, enhancing its performance. 

\subsubsection{Further Analysis} \label{sec:fan}
{\bf \emph{Ablation study}}: For EM-GCN[PD] and EM-GCN[PH], $\lambda_0$ for the delta function is set to be any $\lambda \in \Lambda \backslash \widehat{\Lambda}$. Consequently, $p_0(\cdot)$ will be 0 with probability 1 w.r.t $q(\cdot)$ on $\widehat{\Lambda}$. Hence, for these variants, there is no ``bad'' $\lambda$ such that the corresponding iteration increases ${L}_{\bX}(\cdot,\cdot)$ (cf.\ \cref{rem:badlambda}).

From \cref{results:heterogeneous}, we see that EM-GCN[PT] outperforms EM-GCN[PD], along with EM-GCN[PO] frequently outperforming EM-GCN[PH]. They indicate that increasing the loss for a ``bad'' $\lambda$ is beneficial as it penalizes deviations from desirable graphs.

{\bf \emph{The learned distribution}}: We examine the learned empirical distributions, depicted in \cref{fig:empdistr}. Across all datasets, we notice that the empirical probability of $\lambda$ is relatively high within the range of approximately $[0.4, 0.6]$. This suggests a possible explanation for the decent performance of GCN on a single graph with uniform edge weights.

For IMDB, \cref{eqn:Alambda} is of the form $\lambda A_{\text{MD}} + (1-\lambda)A_{\text{MA}}$. We observe that $\lambda=1$ has a relatively lower probability compared to $\lambda=0$. When $\lambda=1$, it implies that edges in $A_{\text{MA}}$ are all removed. This indicates that the MA relation is more crucial than the MD relation. This observation might be due to MA having an edge density more than triple that of MD. Similarly, for the ACM dataset, where the disparity in edge density is also substantial, \cref{eqn:Alambda} takes the form $\lambda A_{\text{PA}} + (1-\lambda)A_{\text{PS}}$. Here, P, A, S standard for Paper, Author, Subject. The shift of $\lambda$ towards $1$ indicates that the PA relation is more significant than PS, and agrees with the higher density of PA type. The results demonstrate that our approach implicitly captures such key graph features.  
\begin{figure}[!htb]
    \centering
    \begin{subfigure}[b]{0.22\textwidth}
        \includegraphics[width=\textwidth]{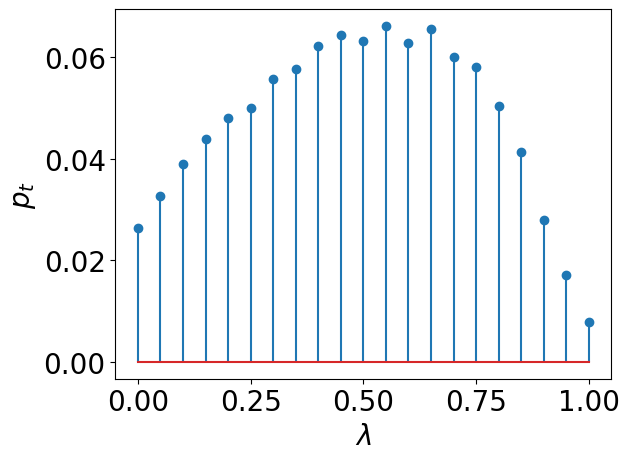}
        \caption{IMDB}
    \end{subfigure}
    \begin{subfigure}[b]{0.22\textwidth}
        \includegraphics[width=\textwidth]{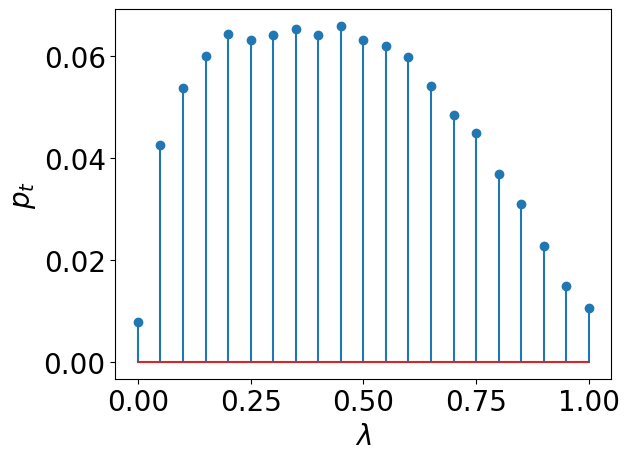}\caption{DBLP}
    \end{subfigure}
    \begin{subfigure}[b]{0.22\textwidth}
        \includegraphics[width=\textwidth]{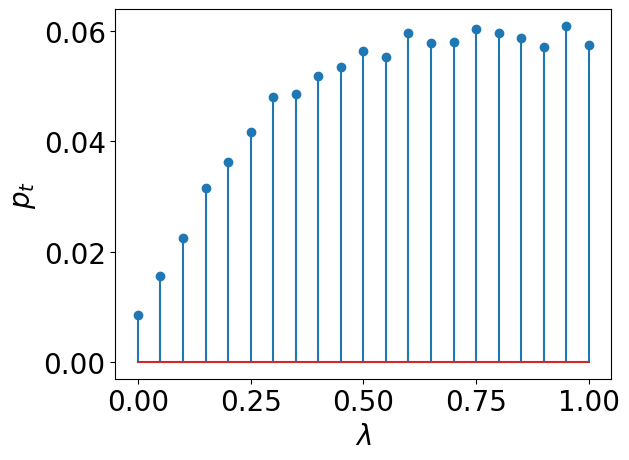}\caption{ACM}
    \end{subfigure}

    \caption{Empirical distribution of $p_t(\cdot)$ from the final E-step.}%
    \label{fig:empdistr}%
\end{figure}
\begin{table*}[!htb]
\centering
\caption{Node classification on homogeneous graphs following the setup of \citet{bgcn}.}\label{tab:homo_result_fixed}
\resizebox{1\textwidth}{!}{
\begin{tabular}{@{}lccccccccc@{}}
\toprule
      & \multicolumn{3}{c}{Cora}                                                    & \multicolumn{3}{c}{Citeseer}                                                & \multicolumn{3}{c}{Pubmed}                                                  \\ 
      & 5 labels                & 10 labels               & 20 labels               & 5 labels                & 10 labels               & 20 labels               & 5 labels                & 10 labels               & 20 labels               \\ \midrule
GCN   & 74.62 $\pm$ 0.54    & 75.30 $\pm$ 0.47          & 81.37 $\pm$ 0.31    & 54.24 $\pm$ 1.26          & 66.07 $\pm$ 0.68          & 70.19 $\pm$ 0.46          & 69.96 $\pm$ 0.65          & 72.96 $\pm$ 0.58          & 78.45 $\pm$ 0.44    \\
DropEdge & 74.79 $\pm$ 0.56 & 75.88 $\pm$ 0.19 & \udcloser{81.46 $\pm$ 0.64} & 54.44 $\pm$ 2.54 & 67.59 $\pm$ 1.41 & 71.23 $\pm$ 1.26& 71.69 $\pm$ 0.50 & 73.14 $\pm$ 0.33 & \udcloser{78.50 $\pm$ 0.54}\\ 
RSGNN & \bf{76.80 $\pm$ 3.19}          & \textbf{78.23 $\pm$ 2.62} & 80.79 $\pm$ 0.91          & \textbf{59.97 $\pm$ 1.89} & 68.41 $\pm$ 0.94         & 69.73 $\pm$ 0.53          & \udcloser{70.45 $\pm$ 0.78}    & 70.92 $\pm$ 0.86          & 77.55 $\pm$ 0.46          \\
BGCN  & \udcloser{75.97 $\pm$ 0.54} & 76.52 $\pm$ 0.50          & 81.18 $\pm$ 0.48          & 56.58 $\pm$ 0.96          & \udcloser{70.61 $\pm$ 0.69}    & \udcloser{72.11 $\pm$ 0.40}    & 70.51 $\pm$ 1.61          & \udcloser{73.36 $\pm$ 1.23}    & 76.55 $\pm$ 0.65          \\ \midrule
EM-GCN & 74.44 $\pm$ 0.76        & \udcloser{76.71 $\pm$ 0.46}    & \textbf{82.24 $\pm$ 0.48} & \udcloser{58.04 $\pm$ 2.24}    & \textbf{70.65 $\pm$ 1.10} & \textbf{72.13 $\pm$ 0.96} & \textbf{74.03 $\pm$ 0.55} & \textbf{74.93 $\pm$ 0.24} & \textbf{78.96 $\pm$ 0.38} \\ \bottomrule
\end{tabular}}
\end{table*} 

\subsection{Homogeneous Graphs} \label{sec:hgr}

\subsubsection{The Experimental Setup and Baslines}
For $G=(V,E)$, we follow \cref{sec:dfd} for the design of $\Lambda$. Specifically, we propose to parametrize graphs by a pair $\bm{\lambda}=\{\lambda_1, \lambda_{2}\}$ with $\lambda_1,\lambda_2 \in [0,0.2]$. Here, $\lambda_1$ is the probability of randomly removing edges from the graph, to account for possible ``noisy edges'' in the observed graphs. On the other hand, $\lambda_2$ is the probability of randomly introducing edges from a pre-constructed subset $E'$ of all missing edges. More specifically, there are $n(n-1)/2$ possible edges between pairs of distinct nodes of size $|V|=n$. For each pair of nodes $v,v'$ without an edge connection in $G$, we compute the cosine similarity of their node features. The edge set $E'$ is obtained from including node pairs whose cosine similarities are above a threshold $\tau$ (see \cref{sec:hda}). It is intuitively considered as the set of  ``likely'' missing edges based on feature similarities. Notice that $\blambda$ does not determine a unique graph, we need to slightly modify \cref{alg:cap2} when applying MCMC (see \cref{sec:hca}).

The most relevant benchmark is BGCN \citep{bgcn} for homogeneous graphs, based on a Bayesian approach to infer the graph distribution. As our construction of $\Lambda$ involves edge removal and addition, we also consider DropEdge \citep{Ron20} and RSGNN \citep{Dai22}, where the former randomly removes edges at each epoch and the latter learns a denser graph using a trained link predictor. The experimental setup follows exactly that from \citet{bgcn}, wherein for each dataset, we evaluate the performance of
the algorithms under limited data scenarios
where only $10$ or $5$ labels per class are available. 

\subsubsection{Results}

Based on \cref{results:heterogeneous} and the ablation study in \cref{sec:fan}, we use the EM-GCN[PT] variant for EM-GCN. Results are shown in \cref{tab:homo_result_fixed}. Overall, across all datasets, EM-GCN surpasses the performance of BGCN and DropEdge and frequently outperforms RSGNN. In principle, DropEdge and RSGNN do not leverage the potential of a distribution of graphs, which is fundamentally different from our approach.

\subsubsection{Further Analysis}
{\bf \emph{Heterophilic graphs}}: 
Recall that a graph is heterophilic if many edges are connecting nodes with different labels. In particular, nodes from the same class are not grouped. Hence, as BGCN has a clustering mechanism, we expect that it might face challenges for heterophilic graphs. On the other hand, based on the construction of $\Lambda$, EM-GCN may generate $\blambda$ such that the associated graph reduces inter-class edges while adding intra-class edges. Hence, we expect EM-GCN should significantly outperform BGCN for heterophilic graphs, which is verified by results in \cref{tab:heterophilic_result}. 

EM-GCN is based on the ``unsuitable'' backbone GCN, which suffers from the same problem as BGCN. Even so, the performance of EM-GCN is much closer to benchmarks ACM-GCN+ \cite{luan2022revisiting} and ACMP \cite{Wan23} dedicated to heterophilic graphs. This validates our proposed construction of $\Lambda$. Furthermore, our framework can be applied to other backbone models and potentially garner performance improvements. As such, we introduce a variant of our model EM-ACM, with ACM-GCN+ as the backbone model. From \cref{tab:heterophilic_result}, we see that EM-ACM generally outperforms its SOTA backbone. It is reasonable to attribute this to the use of a distribution of graphs.   
\begin{table}[!htb]
\centering
\caption{Node classification on heterophilic graphs. The setup and data splits follow \citet{pei2020geomgcn}.}\label{tab:heterophilic_result}
\resizebox{0.88\columnwidth}{!}{
\begin{tabular}{@{}lccc@{}}
\toprule
         & Texas          & Wisconsin      & Cornell        \\ \midrule
GCN      & 55.14 $\pm$ 5.16 & 51.76 $\pm$ 3.06 & 60.54 $\pm$ 5.30 \\
DropEdge & 57.57 $\pm$ 4.94 & 57.45 $\pm$ 5.47 & 60.54 $\pm$ 5.30 \\
RSGNN    &  68.38 $\pm$ 5.26	& 68.82 $\pm$ 7.25	& 60.96 $\pm$ 6.90 \\
BGCN     & 57.96 $\pm$ 6.77 & 61.37 $\pm$ 4.72 & 56.48 $\pm$ 6.67 \\ 
\midrule
ACM-GCN+ & \udcloser{86.76 $\pm$ 4.26} &	\udcloser{86.86 $\pm$ 2.91}	& 84.05 $\pm$ 7.88\\
ACMP & 86.20 $\pm$ 3.00 &	86.10 $\pm$ 4.00	& \bf{85.40 $\pm$ 7.00} 
\\\midrule
EM-GCN    & 79.46 $\pm$ 4.26 & 83.73 $\pm$ 4.34 & 77.30 $\pm$ 4.10 \\ 
EM-ACM & \bf{88.38 $\pm$ 5.14} &	\bf{87.06 $\pm$ 2.51}	& \udcloser{85.14 $\pm$ 6.54} \\ \bottomrule
\end{tabular}}
\end{table}

{\bf \emph{Robustness}}: As discussed in \cref{sec:dfd}, our approach might be resistant to errors in the observed graph, as we are not focusing on a single graph. To verify, we consider the Cora graph $G$ and randomly perturb $r\%$ of edges (by adding new edges and removing existing edges) for $5\leq r\leq 30$. We compare GCN and EM-GCN while observing only the perturbed graph. The results in \figref{fig:perturbation} agree with our speculation that EM-GCN always has a better performance and the gap in accuracies widens as $r$ increases. 

\begin{figure}
    \centering
\includegraphics[width=0.72\columnwidth]{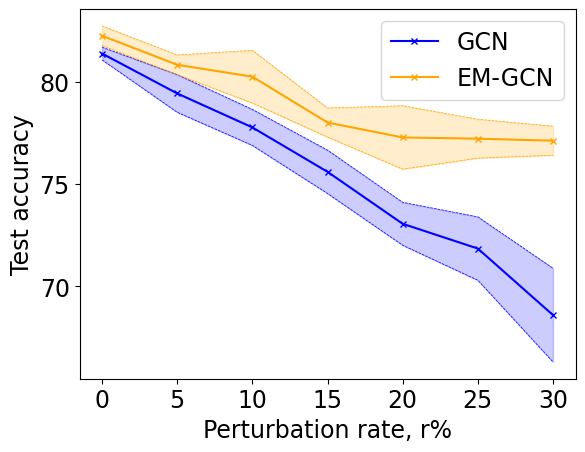}
    \caption{Performance comparison for graph perturbation.}
    \label{fig:perturbation}
\end{figure}

\section{Conclusion}
In this paper, we explore using a distribution of parametrized graphs for training a GNN in an EM framework. Through a probabilistic framework, we handle the uncertainty in graph structures stemming from various sources. Our approach enables the model to handle multiple graphs where the prediction loss is utilized to estimate the likelihood of the graphs. The performance is proved as we provide it with a wider array of graphs, which it can then sift through to acquire more valuable information or remove noise.

\newpage
\section{Impact Statement}
This paper introduces a research work aimed at pushing the boundaries of the Machine Learning field. There are many potential societal consequences of our work, none of which we feel must be specifically highlighted here.

\bibliography{ref}
\bibliographystyle{icml2024}

\newpage
\appendix
\onecolumn



This appendix contains details about notations, datasets, implementation, for chemical datasets, model complexity, and mathematical proofs referenced in the main text of the paper.

\section{List of notations} \label{sec:lon}

For easy reference, we list the most used notations in \cref{tab:lon}.

\begin{table}[!ht]
\caption{List of notations}
\label{tab:lon}
\begin{center}
\scalebox{1}{\begin{tabular}{ c | c } 
  \toprule
  Graph, Vertex set, Edge set & $G \text{ (with subscripts)}, V, E$ \\
 \midrule
   Adjacency matrix & $A \text{ (with subscripts)}$ \\
 \midrule
  Nodes & $v,v'$ \\
 \midrule 
  Shift operator & $S, S_{\lambda}$\\
  \midrule 
  Sample parameter space & $\Lambda$ \\
  \midrule
  Sample parameter & $\lambda \text{ (with subscripts)}$, $\blambda$  \\
  \midrule
  Distributions on the sample space & $\mu \text{ (with subscripts)}$  \\
  \midrule
  Density functions & $p, q$ (with subscripts) \\
  \midrule
  GNN model parameters & $\btheta$ \\
  \midrule
  Features and labels & $\bx, \by, \widehat{\by}, \bX$ \\
  \midrule
  Loss & $\ell(\cdot,\cdot), L_{\bX}(\lambda,\bx;\btheta)$ \\
  \midrule
  GNN model & $\Psi$ \\
  \midrule
  Number of GNN layers & $K$ \\
 \bottomrule
\end{tabular}} 
\end{center}
\end{table}

\section{Graph Regression Task on Chemical Datasets}\label{subsec:chemdataset}

Conventional molecular graph representations mirror a molecule's Lewis structure, with atoms as nodes and chemical bonds as edges. This representation falls short in capturing variations in molecular properties resulting from different three-dimensional (3D) spatial arrangements when molecules share the same topology, as seen in cases like cis-trans isomers. Moreover, molecules inherently possess uncertainty due to their quantum mechanical properties, particularly concerning electron orbitals. Hence, using a distribution of graphs for learning in such cases is a sensible choice.

The process of generating different molecular graphs begins with the acquisition of coarse 3D coordinates of the atoms in a molecule using RDKit\footnote{A cheminformatics tool. https://www.rdkit.org}. Following that, the interatomic Euclidean distances between all atoms within the molecule are calculated. A parameter, $\lambda$ is then introduced to define a threshold range, $[0, \lambda]$, for determining node connections and generating multiple graph instances. The notion of employing thresholding based on the interatomic distance between nodes in molecular graphs has been previously documented in works such as \cite{hmg2020} and \cite{shen2023molecular}. The former introduced a cut-off distance hyperparameter to construct heterogeneous molecular graphs. Meanwhile, our approach aligns more closely with the latter, where the Vietoris-Rips complex and thresholding are used to form a series of $G_\lambda$ graphs. However, in \cite{shen2023molecular}, they utilized five non-overlapping, manually adjusted intervals for thresholding and adopted a computationally intensive multi-channel configuration to learn from the five generated graphs. There are also works such as \cite{thomas2018tensor}, where molecules are treated as 3D point clouds and a radius is set to specify interacting vertices. 

The graph construction process may involve adding new edges, connecting distant nodes, or removing existing edges. Ideally, the model should prioritize ``useful'' graph realizations and assign a low probability to less beneficial ones, effectively discarding them.

\begin{table*}[!htb]
\centering
    \caption{Graph regression task on molecular datasets. Average test rmse reported, the lower the better.}\label{tab:result_chem}
\vspace{2mm}
\resizebox{0.95\textwidth}{!}{
\begin{tabular}{ccccccc}\toprule
            & \multicolumn{3}{c}{Random split}                                                        & \multicolumn{3}{c}{Scaffold split}                                          \\
Datasets     & FreeSolv                 & ESOL                      & Lipophilicity              & FreeSolv                 & ESOL                     & Lipophilicity   \\ \midrule
GCN         & \udcloser{1.157 $\pm$ 0.215}          & \udcloser{0.652 $\pm$ 0.073}           & 0.707 $\pm$ 0.030            & \udcloser{2.618 $\pm$ 0.298}          & 0.876 $\pm$ 0.037          & 0.760 $\pm$ 0.009 \\
GAT         & 1.873 $\pm$ 0.522          & 0.837 $\pm$ 0.101           & 0.704 $\pm$ 0.058            & 2.942 $\pm$ 0.591          & 0.907 $\pm$ 0.034          & 0.777 $\pm$ 0.037 \\
Weave       & 1.497 $\pm$ 0.251          & 0.798 $\pm$ 0.088           & 0.789 $\pm$ 0.059            & 3.129 $\pm$ 0.203          & 1.104 $\pm$ 0.063          & 0.844 $\pm$ 0.031 \\
MPNN        & 1.388 $\pm$ 0.404          & 0.703 $\pm$ 0.075           & \udcloser{0.640 $\pm$ 0.025}            & 2.975 $\pm$ 0.775          & 1.117 $\pm$ 0.058          & \bf{0.735 $\pm$ 0.019} \\
AttentiveFP & 1.275 $\pm$ 0.289          & 0.673 $\pm$ 0.085           & 0.719 $\pm$ 0.042            & 2.698 $\pm$ 0.297          & \udcloser{0.855 $\pm$ 0.029} & 0.762 $\pm$ 0.022 \\
GIN         & 1.678 $\pm$ 0.494          & 0.792 $\pm$ 0.097           & 0.716 $\pm$ 0.073            & 2.957 $\pm$ 0.696          & 0.990 $\pm$ 0.057          & 0.770 $\pm$ 0.021 \\ \midrule
EM-GCN$^*$  & \bf{0.936 $\pm$ 0.162} & \bf{0.606 $\pm$ 0.041} & \bf{0.639 $\pm$ 0.028} & \bf{2.189 $\pm$ 0.128} & \bf{0.834 $\pm$ 0.027}    &   \udcloser{0.743 $\pm$ 0.013}     \\   \bottomrule 
\end{tabular}}
\end{table*}

\subsection{Baselines and Datasets}
MoleculeNet\footnote{https://moleculenet.org/datasets-1} is a popular benchmark for molecular machine learning, encompassing multiple datasets to be tested on a diverse of molecular properties. For our evaluation, we specifically selected the datasets FreeSolv, ESOL, and Lipophilicity, all of which are designed for graph regression tasks.

We compared our approach against standard models for molecular properties prediction that do not incorporate transfer learning from a larger dataset such as Zinc15\footnote{A database of purchasable drug-like compounds; https://zinc.docking.org/tranches/home/}. The selected baseline models for this comparison included Weave \citep{Kearnes_2016}, MPNN \citep{mpnn2017}, AttentiveFP \citep{attentivefp}, GIN \citep{xu2018gin}, as well as the standard GCN and GAT models. For EMGNN, a GNN model that generalizes GCN with a degree-1 convolutional filter $P_\lambda(S_\lambda)$ (refer to \cref{sec:spo}) is utilized as the backbone of our model. As such, we name the resulting model EM-GCN$^*$. The sample space $\Lambda$ spans the range $[1,10]\text{\AA}$ and $\widehat{\Lambda}$ is the discretized space with $0.05$ increments.

\subsection{Experimental Results}

In \cref{tab:result_chem}, the average test root mean square error (rmse) over ten runs with standard deviation is reported for the graph regression task, where the molecular properties of molecular graphs are to be predicted. The result shown is for the case of $q(\cdot)=p_t(\cdot)$. We observe that EM-GCN$^*$ frequently performed better than the baselines. This may be due to the training process of EM-GCN$^*$, which exposes it to diverse graph realizations, allowing it to capture non-covalent interactions that are critical for characterizing the physical properties of molecules. In contrast, the baselines employ the conventional molecular graph representation. We note that our framework does not explicitly incorporate bond angles but it does expose the model to graphs with a broad range of connectivities. This exposure indirectly integrates geometric information, as the latent variable constructs graphs with bond lengths falling within specific ranges. This provides our model with additional 2D information regarding interatomic distances, which may offer insights into the underlying 3D structure.

\section{Datasets and Implementation Details} \label{sec:dod}
All datasets used in this paper are publicly available and open-source.

\subsection{Heterogeneous Datasets} \label{sec:hed}
The characteristics of the three heterogeneous benchmark datasets are as summarized in \cref{table:hetero_dataset}. In the DBLP \url{https://dblp.uni-trier.de/} dataset, the research areas of the authors are to be predicted. Meanwhile, in the ACM \url{http://dl.acm.org/} and IMDB \url{https://www.imdb.com/interfaces/} datasets, the categories of papers and genres of movies are to be determined, respectively. The datasets are publicly available at \url{https://github.com/seongjunyun/Graph\_Transformer\_Networks}.

In the DBLP dataset, there are three distinct node types (Paper (P), Author (A), Conference (C)) and two edge types (PA, PC). The ACM dataset also comprises three node types (Paper (P), Author (A), and Subject (S)) and two edge types (PA, PS). 

\begin{table}[!htb]
\centering
\captionsetup{justification=centering}
\caption{Heterogeneous datasets. The number of A-B edges is equal to the number of B-A edges thus omitted.}
\label{table:hetero_dataset}
\vspace{2mm}
\begin{tabular}{@{}cccccc@{}}
\toprule
                      & Edge types (A-B)   & \# of A-B & \# Edges               & \# Edge types      & \# Features           \\ \midrule
\multirow{2}{*}{DBLP} & Paper - Author     & 19645     & \multirow{2}{*}{67946} & \multirow{2}{*}{4} & \multirow{2}{*}{334}  \\
                      & Paper - Conference & 14328     &                        &                    &                       \\ \midrule
\multirow{2}{*}{ACM}  & Paper - Author     & 9936      & \multirow{2}{*}{25922} & \multirow{2}{*}{4} & \multirow{2}{*}{1902} \\
                      & Paper - Subject    & 3025      &                        &                    &                       \\ \midrule
\multirow{2}{*}{IMDB} & Movie - Director   & 4661      & \multirow{2}{*}{37288} & \multirow{2}{*}{4} & \multirow{2}{*}{1256} \\
                      & Movie - Actor     & 13983     &                        &                    &                       \\ \bottomrule
\end{tabular}
\end{table}

\subsection{Homogenous Datasets} \label{sec:hda}
The utilized homogeneous graph datasets include Cora, Citeseer, Pubmed, Texas, Cornell and Wisconsin. The first three mentioned datasets citation networks where the nodes represent documents and the edges denote citation linkages \cite{kipf2016semi}. On the other hand, the latter three datasets are heterophilic, featuring connections between dissimilar nodes. We consider the edges as undirected. The statistics and parameters of the datasets are summarized in \cref{table:homo_dataset}.

\begin{table}[!htb]
\centering
\captionsetup{justification=centering}
\caption{Statistics and parameters of homogeneous graph datasets.}
\label{table:homo_dataset}
\vspace{2mm}
\begin{tabular}{@{}cccccc@{}}
\toprule
         & \# Nodes & \# Edges & \# Features & \# Classes & Threshold $\tau$ \\ \midrule
Cora     & 2708     & 5429     & 1433        & 7   &    0.5   \\
Citeseer & 3327     & 4732     & 3703        & 6   &    0.6   \\
Pubmed   & 19717    & 44338    & 500         & 3    &  0.9    \\ 
Texas    & 183     & 295     & 1703        & 5    &   0.6   \\
Cornell  & 183     & 280     & 1703        & 5     &   0.6  \\
Wisconsin & 251     & 466     & 1703        & 5    &    0.6  \\\bottomrule
\end{tabular}
\end{table}

\subsection{Chemical Datasets} 

The selected datasets, FreeSolv, ESOL, and Lipophilicity, are intended for graph regression tasks. In FreeSolv, the task involves estimating solvation energy. In ESOL, the objective is to estimate solubility. Lastly, Lipophilicity is for the estimation of lipophilicity. These are important molecular properties in the realm of physical chemistry and provide insights into how molecules interact with solvents.

\begin{itemize}
  \item FreeSolv. The dataset contains experimental and calculated hydration-free energy of 642 small neural molecules in water. 
  \item ESOL (for \emph{E}stimated \emph{SOL}ubility). ESOL consists of water solubility data for 1128 compounds. 
  \item Lipophilicity. This dataset contains experimental results of the octanol/water distribution coefficient of 4200 compounds, namely logP at pH 7.4, where P refers to the partition coefficient. Lipophilicity refers to the ability of a compound to dissolve in fats, oils and lipids. 
\end{itemize}

\subsection{Hyperparameters, Code and Implementation Tips} \label{sec:hca}
The models were trained on a server equipped with four NVIDIA RTX A5000 GPUs for hardware acceleration. As the parameter spaces are small for the datasets used in the paper, we apply the standard Metropolis-Hastings algorithm for MCMC. The code is uploaded as \emph{supplementary materials}. 

\paragraph{Heterogeneous node classification.} The Simple-HGN model was obtained from the CogDL library \citep{cen2023cogdl}, SeHGNN from their code repository (\url{https://github.com/ICT-GIMLab/SeHGNN}). Meanwhile, GAT was sourced from the DGL library \citep{wang2019dgl} and GCN is from \url{https://github.com/tkipf/pygcn}.

The hidden units are set to 64 for all models. The hyperparameters of SeHGNN and Simple-HGN are as in the respective repository. For fair comparison, GAT and GCN are evaluated on the entire graph ignoring the types. Specifically for our model, the number of MCMC iterations $N_{mc}$ is set to be $15000$, the $T$ denoting the number of EM iterations and $T'$ is tuned by searching on the following search spaces: $[10, 15, 20, 25, 30]$. 

\paragraph{Homogeneous node classification.} The BGCN model was obtained from its code repository at \url{https://github.com/huawei-noah/BGCN/tree/master}. Likewise, for RSGNN, the model was sourced from its code repository located at \url{https://github.com/EnyanDai/RSGNN}. As for the ACM-GCN+ model, it was acquired from its code repository (\url{https://github.com/SitaoLuan/ACM-GNN}). The hyperparameters used for evaluating the models align with those specified in their respective repositories if provided.

Unlike heterogeneous graphs, during implementation, once $\blambda = (\lambda_1,\lambda_2)$ is determined in MCMC, we still need to draw a graph $G_{\blambda}$ according to $\blambda$. The graph $G_{\blambda}$ is stored for later use. 

In the context of our model, $N_{mc}$ is set to be 15000, the $T$ and $T'$ is tuned by searching on the following search spaces: $[10, 15, 20, 25, 30]$. The choice of the threshold $\tau$ (see \cref{sec:hgr}) is tabulated in \cref{table:homo_dataset}. Notice that the threshold $\tau$ is chosen such that $E'$ for each dataset is neither too small or too big. 

\paragraph{Graph regression on molecular datasets.} The experiments on this task were facilitated using the DGL-LifeSci library \citep{dgllife} which provided the code and hyperparameters for the baseline models. For fair comparisons, all the models were evaluated using the canonical features as documented at \url{https://lifesci.dgl.ai/generated/dgllife.utils.CanonicalAtomFeaturizer.html}. Specifically for our model,  $N_{mc}$ is set to 18000, $T$ and $T'$ is tuned by searching on the following search spaces: $[10, 15, 20, 25, 30]$.

\section{Complexity} \label{sec:com}
We discuss the complexity of EM-GCN, focusing on the steps involving MCMC. The complexity of the remaining steps depends mainly on the chosen base model, e.g., GCN. We mainly compare with BGCN and RSGNN, which share certain common features with our model. 

In the E-step, where MCMC is executed, the model parameters remain fixed. Identical inputs to the ``frozen'' model consistently produce the same $L_\bX(\blambda, \btheta)$. Hence, to expedite MCMC iterations for a specific $\btheta$, we precompute, store and reuse $L_\bX(\blambda, \btheta), G_{\blambda}$ for all $\blambda \in \Lambda$, using a space complexity of $O(|\Lambda|(|E| + |V|))$, compared to BGCN's $O(|E|+|V|)$ and RSGNN's $O(|V|^2)$. Therefore, the space complexity of EM-GCN is higher than that of BGCN in general. It is lower than RSGNN for sparse graphs if $|V| \gg |\Lambda|$, i.e., the size of the (discretized) sample space is much smaller than the number of nodes.

As such, our model takes a comparatively lower training time than RSGNN and BGCN. The theoretical time complexity of our model is $O(K|\Lambda|\cdot|E| + N_{mc})$ for $N_{mc}$ MCMC iterations and $K$ GCN layers (we consider only the message passing operation in GCN layers). On the other hand, the time complexity of RSGNN with $\kappa$ MLP layers is $O(|V|d'd + (\kappa-1) N d^2 + K|E|)$ and BGCN is $O(K|E| + N_b |V|^2 N_c)$ for $N_b$ MMSBM iterations, $N_c$ classes, $d'$ feature size and $d$ hidden dimension. Time complexity does not include extra time taken for backpropagation in models with more parameters like BGCN and RSGNN. 

We show the explicit run-times (in seconds) for the models compared in \cref{tab:rtc}, with the same software and hardware configurations. We see that EM-GCN is much faster than both BGCN and RSGNN.

\begin{table}[!htb]
\centering
    \caption{Run-time (in sec.) comparison among different models }\label{tab:rtc}
\vspace{2mm}
\resizebox{0.5\textwidth}{!}{
\begin{tabular}{@{}cccc@{}}
\toprule
    & Cora              & Citeseer              & Pubmed            \\ \midrule
BGCN  & 328.14 $\pm$ 2.18 & 323.55 $\pm$ 4.86 & 1446.91 $\pm$ 31.31 \\
RSGNN & 147.09 $\pm$ 4.68 & 164.83 $\pm$ 1.71 & 341.52 $\pm$ 3.06  \\
EM-GCN & 86.06 $\pm$ 1.39 & 65.19 $\pm$ 14.27 & 102.18 $\pm$ 2.63  \\ \bottomrule
\end{tabular}}
\end{table}

\section{Related Work}\label{sec:related_works}

This section provides an overview of related work, specifically on the selected baseline models and their relevance to our model. On the aspect of heterogeneous graphs, GTN \cite{yun2019gtn} emphasized that disregarding the type information of such graphs and treating them as homogeneous is suboptimal. To address this, they devised a method to learn new graph structures using multiple candidate (edge type-specific) adjacency matrices and attention mechanisms. Their work is relevant to ours because, akin to GTN, we consider the given graph suboptimal. However, instead of learning new meta-path-related graph structures that introduce connections between unconnected nodes in the original graph, our setup and approach focus on the information transmission rates in such graphs. 

On the other hand, concerning homogeneous graphs, given that our model focuses on learning from a distribution of graphs, it makes sense to compare against methods that alter the observed one, albeit in different ways. As such, DropEdge \cite{Ron20} and RSGNN \cite{Dai22} were considered. In the case of DropEdge, edges in the given graph were randomly dropped at a manually selected rate in each epoch. This process generated different copies of the original graph, diversifying the input data used to train the model. Consequently, DropEdge was deemed capable of overcoming overfitting and oversmoothing. 

Meanwhile, for RSGNN, the model trained a link predictor to rewire the observed graph, eliminating or down-weighting noisy edges (connecting nodes with dissimilar features) and adding edges to densify the graphs. This was done to ensure the model is robust to noise and to address the issue of label sparsity by connecting more unlabeled nodes to labeled ones. In \cite{jifeng23}, a two-stage scheme is proposed to enhance the performance of existing GNN models, where the second stage involves edge drop along estimated class boundaries. 

Our work also shares a close connection with BGCN \cite{bgcn}, a work employing a Bayesian formulation tailored to address uncertainty in graphs. In their framework, the model parameters were treated as random variables, incorporating a prior distribution. Additionally, BGCN interpreted the observed graph as an instantiation from a parametric family of random graphs. Specifically, it adopted the assortative mixed membership stochastic block model (aMMSBM) as the underlying graph model. 

\section{Theoretical Discussions}
\label{sec:tdi}
\subsection{Discussions and the Proof of \cref{thm:iit}} \label{sec:dat}

Assume the parameter sample space $\Lambda$ is a metric space, whose metric is $d(\cdot,\cdot)$. Define the \emph{Wasserstein space} $\calP(\Lambda)$ to be the space of (Borel) probability distributions on $X$ with finite mean and variance. 

Given $\mu_1,\mu_2$ in $\calP(\Lambda)$, the \emph{Wasserstein metric} $W(\mu_1,\mu_2)$ is defined by
\begin{align*}
    W(\mu_1,\mu_2)^2 = \inf_{\gamma\in \Gamma(\mu_1,\mu_2)}\int d (\lambda,\lambda')^2 \ud\gamma(\lambda,\lambda'),
\end{align*}
where $\Gamma(\mu_1,\mu_2)$ is the set of \emph{couplings} of $\mu_1,\mu_2$, i.e., the collection of probability measures on $\Lambda \times \Lambda$ whose marginals are $\mu_1$ and $\mu_2$, respectively.

Intuitively, the Wasserstein metric is the minimum amount of ``work'' required to transform one probability distribution into the other, where the ``work'' is the sum of the product of the amount of probability mass to be moved and the distance that it must be moved. It is well-known that $W(\cdot,\cdot)$ makes $\calP(\Lambda)$ a metric space \citet{Vil09, Jif23, Ji23}.

Recall that for each $\lambda\in \Lambda$ there is an associated graph shift operator $S_{\lambda} \in M_n(\mathbb{R})$, where $n$ is the size of the node set $V$. This means that we have a \emph{parameterization map} 
\begin{align*}
\mathfrak{p}: \Lambda \to M_n(\mathbb{R}), \lambda\mapsto S_{\lambda},
\end{align*}
and $M_n(\mathbb{R})$ is endowed with the operator norm denoted by $\norm{\cdot}$. 

We analyze the following fundamental GNN model. For each graph shift operator $S_{\lambda}$, it defines a GNN layer according to (\ref{eq:vsw}), where the aggregation ``AGGR'' is achieved by multiplying the input features by $S_{\lambda}$. We further assume that the matrix $W^k$ in (\ref{eq:vsw}) depends only on $\btheta$ and the ``ReLU'' function is used for the activation $\sigma$. For a $K$, let $\Psi(\lambda,\cdot; \btheta)$ be the resulting $K$-layer GNN, where $\btheta$ is the estimated parameters of the model. For any input features $\bx$, we have the expected output $\bz_{\mu} = \E_{\lambda \in \mu}[\Psi(\lambda, \bx;\btheta)]$. Therefore, assuming $\bx, \btheta$ are fixed, we have the following \emph{feature map}
\begin{align*}
   \mathfrak{f}: \calP(\Lambda) \to \mathbb{R}^n, \mu \mapsto \bz_{\mu} = \E_{\lambda \in \mu}[\Psi(\lambda, \bx;\btheta)].
\end{align*}

In general, if $\phi: X_1 \to X_2$ is a map between metric spaces with respective metrics $d_1(\cdot,\cdot)$ and $d_2(\cdot,\cdot)$, then it is called \emph{H\"{o}lder continuous} if for some constants $C, \alpha>0$, we have 
$d_2\big(\phi(x_1),\phi(x_2)\big) \leq Cd_1(x_1,x_2)^{\alpha}, x,y \in X_1$. Ignoring the scalar $C$, we may also say that $\phi$ is \emph{$\alpha$-H\"{o}lder continuous} to emphasize the exponent. 

We can now state the following precise version of \cref{thm:iit}. 

\begin{Theorem}
Let $K$ be the number of GNN layers. If the parameterization $\mathfrak{p}: \Lambda \to M_n(\mathbb{R})$ is $\alpha$-H\"{o}lder continuous and bounded, then the feature map $\mathfrak{f}: \calP(\Lambda) \to \mathbb{R}^n$ is $\beta$-H\"{o}lder continuous for $\beta = K\alpha/(K\alpha+2)$. Moreover, if the activation $\sigma$ is bounded, e.g., $\tanh$, sigmoid, then we can remove the condition that $\mathfrak{p}$ is bounded.
\end{Theorem}

\begin{proof}
As $\mathfrak{p}$ is H\"{o}lder continuous, there is $C, \alpha>0$ such that $\norm{S_{\lambda} -S_{\lambda'}} \leq Cd(\lambda,\lambda')^{\alpha}$. Moreover, by the assumption, there is an upper bound $B$ on the norms of the image of $\mathfrak{p}$. 

For each $\lambda \in \Lambda$, the $k$-th GNN layer for $k\leq K$ takes the form $\by \mapsto \sigma(S_{\lambda}\by W^k)$, where $W^k$ is determined by the fixed model parameters $\btheta$. As the activation function $\sigma$ is $1$-Lipshitz, we have 
\begin{align} \label{eq:nss}
\norm{\sigma(S_{\lambda}\by W^k) - \sigma(S_{\lambda'}\by W^k)} \leq \norm{W^k}\norm{\by} \norm{S_{\lambda} - S_{\lambda'}} \leq C_1 \norm{\by} d(\lambda, \lambda')^{\alpha}, 
\end{align}
where $C_1$ is the constant $C\norm{W^k}$. Taking $\by$ as the output of the previous layer, we may repeat (\ref{eq:nss}) and obtain 
\begin{align}\label{eq:npl}
    \norm{\Psi(\lambda, \bx;\btheta) - \Psi(\lambda', \bx;\btheta)} \leq C_2d(\lambda,\lambda')^{K\alpha},
\end{align}
where $C_2$ is a constant independent of $\lambda, \lambda'$. 

Moreover, due to the norm upper bound $B$ of $S_{\lambda}$ for any $\lambda \in \Lambda$, there is an upper bound $B_1$ of $\Psi(\lambda, \bx;\btheta)$ independent of $\lambda$. If the activation $\sigma$ is bounded, then the same boundedness conclusion holds without assuming $\mathfrak{p}$ is bounded as the last layer ends with the bounded function $\sigma$.

To prove the theorem, consider $\mu_1, \mu_2 \in \calP(\Lambda)$. Let $\gamma \in \Gamma(\mu_1,\mu_2)$ be a distribution on $\Lambda \times \Lambda$ such that 
\begin{align} \label{eq:wmm}
        W(\mu_1,\mu_2)^2 = \int d (\lambda,\lambda')^2 \ud\gamma(\lambda,\lambda').
\end{align}
As the marginals of $\gamma$ are $\mu_1$ and $\mu_2$ respectively, we may express $\norm{\mathfrak{f}(\mu_1)-\mathfrak{f}(\mu_2)}$ as
\begin{align*}
    &\norm{\mathfrak{f}(\mu_1)-\mathfrak{f}(\mu_2)}  = \norm{\E_{\lambda \in \mu_1}[\Psi(\lambda, \bx;\btheta)]- \E_{\lambda \in \mu_2}[\Psi(\lambda, \bx;\btheta)]} \\
    = & \norm{\int \Psi(\lambda, \bx;\btheta) \ud \mu_1(\lambda) - \int \Psi(\lambda', \bx;\btheta) \ud \mu_2(\lambda')} \\
    = & \norm{\int \Psi(\lambda, \bx;\btheta) - \Psi(\lambda', \bx;\btheta) \ud \gamma(\lambda,\lambda')}.
\end{align*}
Let $a > 0$ be a number to be determined later. By (\ref{eq:wmm}), the Chebyshev inequality implies that the $\gamma$-measure of the set $\Sigma_a =\{(\lambda,\lambda') \mid d(\lambda,\lambda')\geq a\}$ is bounded by $W(\mu_1,\mu_2)/a^2$. With this, we estimate $\norm{\mathfrak{f}(\mu_1)-\mathfrak{f}(\mu_2)}$ as follows
\begin{align*}
    &\norm{\mathfrak{f}(\mu_1)-\mathfrak{f}(\mu_2)}  \leq  \int \norm{\Psi(\lambda, \bx;\btheta) - \Psi(\lambda', \bx;\btheta)} \ud \gamma(\lambda,\lambda') \\
    = & \int_{\Sigma_a} \norm{\Psi(\lambda, \bx;\btheta) - \Psi(\lambda', \bx;\btheta)} \ud \gamma(\lambda,\lambda') + \int_{\Lambda\times \Lambda \backslash \Sigma_a} \norm{\Psi(\lambda, \bx;\btheta) - \Psi(\lambda', \bx;\btheta)} \ud \gamma(\lambda,\lambda') \\
    \leq & \int_{\Sigma_a} 2B_1 \ud \gamma(\lambda,\lambda') + \int_{\Lambda\times \Lambda \backslash \Sigma_a} C_2d(\lambda,\lambda')^{K\alpha} \ud \gamma(\lambda,\lambda') \\
    \leq & \frac{2B_1W(\mu_1,\mu_2)}{a^2} + C_2a^{K\alpha}. 
\end{align*}
Minimizing the right-hand-side (by taking its derivative w.r.t.\ $a$), we have $a = C_3W(\mu_1,\mu_2)^{1/(K\alpha+2)}$, where $C_3 = [4B_1/(C_2K\alpha)]^{1/(K\alpha+2)}$. Plugging in the expression of $a$ into the estimation, we have
\begin{align*}
\norm{\mathfrak{f}(\mu_1)-\mathfrak{f}(\mu_2)} \leq \frac{2B_1}{C_3^2} W(\mu_1,\mu_2)^{1-2/(K\alpha+2)}+ C_2C_3^{K\alpha} W(\mu_1,\mu_2)^{K\alpha/(K\alpha+2)}.
\end{align*}
Therefore, $\mathfrak{f}$ is $\beta$-H\"{o}lder continuous for $\beta = K\alpha/(K\alpha+2)$. 
\end{proof}

\subsection{The Proof of \cref{thm:afa}} \label{sec:tpo}

The proof is a refinement of the proof of \citep[Theorem 3.12]{Har16}. As $L_{\bX}$ is assumed to be convex, we observe that in the expression for $J_{\Lambda_{T'}}(\widehat{\btheta})$, a term is convex if $p(\lambda_i) \geq p_0(\lambda_i)$ and concave otherwise. Therefore, we need to separate these two cases when performing gradient descent. 

We follow the proof of \citep[Theorem 3.12]{Har16}, and highlight necessary changes. By \citep[Theorem 2.2]{Har16}, it suffices to show that the algorithm $\mathcal{A}$ is $\epsilon$-uniformly stable \citep[Definition 2.1]{Har16}. Let $\Lambda_1 = \{\lambda_{1,1},\ldots,\lambda_{1,T'}\}$ and $\Lambda_2 = \{\lambda_{2,1},\ldots,\lambda_{2,T'}\}$ be two sample sequences of $\lambda$ that differ in only a single sample. Consider the gradient updates $\Gamma_{1,1}, \ldots, \Gamma_{1,T'}$ and $\Gamma_{2,1}, \ldots, \Gamma_{2,T'}$. Let $\widehat{\btheta}_{1,t'}$ and $\widehat{\btheta}_{2,t'}, t'\leq T'$ be the corresponding outputs of the algorithm $\mathcal{A}$. 

Introduce $f(\lambda,\btheta)=\frac{p(\lambda)-p_0(\lambda)}{q(\lambda)}L_{\bX}(\lambda,\btheta)$. As $L_{\bX}(\lambda,\cdot)$ is convex, $\alpha$-Lipschitz and $\beta$-smooth, $f(\lambda,\btheta)$ is $\alpha\gamma$-Lipschitz, $\beta\gamma$-smooth. Moreover, it is convex if $p(\lambda) \geq p_0(\lambda)$. 

Write $\delta_{t'}$ for $\norm{\widehat{\btheta}_{1,t'}-\widehat{\btheta}_{2,t'}}$. Using the fact that $f(\lambda,\cdot)$ is $(\alpha\gamma)$-Lipschitz, by \citep[Lemma 3.11]{Har16}, we have for any $t_0\leq T'$
\begin{align*}
    &\mathbb{E} \big(|f(\lambda,\widehat{\btheta}_{1,T'})-f(\lambda,\widehat{\btheta}_{2,T'})|\big) \\ & \leq \frac{t_0}{T'} + \alpha\gamma\mathbb{E}(\delta_{T'}\mid \delta_{t_0}=0). 
\end{align*}
We need an estimation of $\mathbb{E}(\delta_{T'}\mid \delta_{t_0}=0)$. For convenience, for any $t'\geq t_0$, let $\Delta_{t'} = \mathbb{E}(\delta_{t'}\mid \delta_{t_0}=0)$. 

Observe that at step $t'$, with probability $1-1/n$, the samples selected are the same in both $\Lambda_1$ and $\Lambda_2$. Moreover, with probability $(1-1/n)b_1$, the common sample is convex, so we can apply \citep[Lemma 3.7.2]{Har16}. With probability $(1-1/n)(1-b_1)$, the common sample is non-convex, and \citep[Lemma 3.7.1]{Har16} is applicable. With probability $1/n$, the selected samples are different and $\Gamma_{1,t'}$ and $\Gamma_{2,t'}$ are respectively $\frac{|p(\lambda)-p_0(\lambda)|}{q(\lambda)}\alpha a_{t'}$-bounded, by \citep[Lemma 3.3]{Har16}. 

Therefore, by linearity of expectation and \citep[Lemma 2.5]{Har16}, we may estimate:
\begin{align*}
    \Delta_{t+1} & \leq (1-\frac{1}{n})\big(b_1 + (1-b_1)(1+a_{t'}\beta\gamma)\big)\Delta_{t'} + \frac{1}{n}\Delta_{t'} \\ & + \frac{\alpha a_{t'}}{n}\mathbb{E}_{\lambda_{1,t'},\lambda_{2,t'}}\big( \frac{|p(\lambda_{1,t'})-p_0(\lambda_{1,t'})|}{q(\lambda_{1,t'})} \\ & + \frac{|p(\lambda_{2,t'})-p_0(\lambda_{2,t'})|}{q(\lambda_{2,t'})} \big) \\
    & = (1-\frac{1}{n})\big(b_1 + (1-b_1)(1+a_{t'}\beta\gamma)\big)\Delta_{t'} \\& + \frac{1}{n}\Delta_{t'} + \frac{2b_2\alpha a_{t'}}{n} \\
    & \leq (1-\frac{1}{n})\big(b_1 + (1-b_1)(1+\frac{c\beta\gamma}{t'})\big)\Delta_{t'} \\&  + \frac{1}{n}\Delta_{t'} + \frac{2b_2\alpha c}{nt'} \\ & = \big(1 + (1-\frac{1}{n})(1-b_1)\frac{c\beta\gamma}{t'} \big) \Delta_{t'} + \frac{2c(b_2\alpha)}{nt'} \\
    & \leq \exp\big((1-\frac{1}{n})(1-b_1)\frac{c\beta\gamma}{t'}\big)\Delta_{t'} + \frac{2c(b_2\alpha)}{nt'} \\
    & = \exp\Big((1-\frac{1}{n})\frac{c\big((1-b_1)\beta\gamma\big)}{t'}\Big)\Delta_{t'} + \frac{2c(b_2\alpha)}{nt'}. 
\end{align*}
Then, by the same argument as in the proof of \citep[Theorem 3.12]{Har16}, we have the algorithm $\mathcal{A}$ is $\epsilon$-uniformly stable for any
\begin{align*}
\epsilon \leq C \big(\frac{b_2^2\alpha^2}{T'}\big)^{\frac{1}{\beta\gamma c(1-b_1)+1}},
\end{align*}
for some constant $C$ independent of $T'$ and $\alpha$. 

\section{Miscellaneous Discussions}

\subsection{The Intuition of Typicality (cf.\ (\ref{eq:elb}))} \label{sec:tio}
\begin{figure}[htb!]
    \centering
    \includegraphics[scale=0.5]{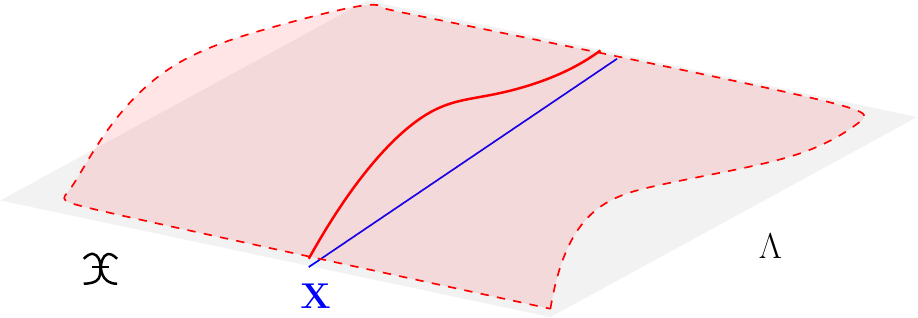}
    \caption{An illustration of the typicality assumption.}
    \label{fig:aio}
\end{figure}
As illustrated in \figref{fig:aio}, we are interested in taking the average, formally the expectation w.r.t.\ a distribution $\pi_0$, of a function over $\Lambda \times \mathfrak{X}$. The function is illustrated by the red surface in \figref{fig:aio}. The ``typicality'' assumption on $\bX \in \mathfrak{X}$ requires that there is a distribution $p_{0,\bX}$ such that the above average over $\Lambda \times \mathfrak{X}$ is (approximately) the same as the average, i.e., expectation w.r.t.\ $p_{0,\bX}$, over the ``line'' $\{\bX\} \times \Lambda$. Intuitively, ``typicality'' says that the cross-section of the function at $\bX$ displays the same pattern as the function on the entire domain $\bLambda\times \mathfrak{X}$. Hence, statistics on the entire domain can be estimated from observations on the cross-section $\{\bX\} \times \Lambda$.

As we have mentioned in \cref{rmk:fta}, ``typicality'' allows us to simplify the optimization problem. For future work, we are also interested in solving the optimization with a weaker assumption.

\subsection{The Dominant Component of \cref{eq:lca}}\label{subsubsec:numverify}
\begin{figure}[!htb]
    \centering
\includegraphics[width=0.5\columnwidth]{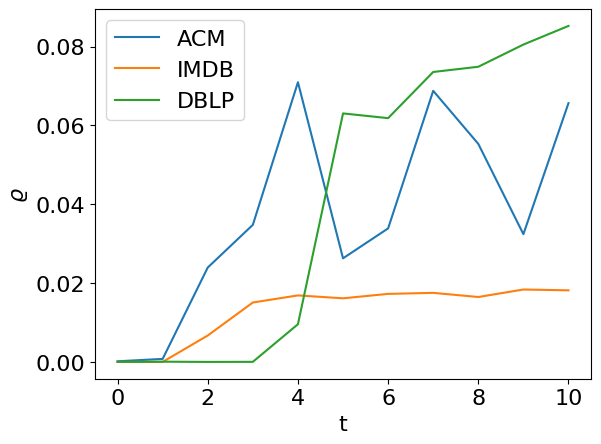}
    \caption{Plot of $\varrho$ across $t$ EM iterations}\label{fig:ratio}
\end{figure}
In \cref{subsec:EM}, we use $-\eta^{(t)}\E_{\lambda\sim p_0}\big({L}_{\bX}(\lambda,\btheta) \big)$ to approximate $\log C(\btheta)$. To justify this, we provide numerical evidence that the dominant component of \cref{eq:lca} is $-\eta^{(t)}\E_{\lambda\sim p_0}\big({L}_{\bX}(\lambda,\btheta) \big)$, based on \cref{subsec:heterograph}. The above assertion is supported by assessing the ratio
\begin{align*}
     \varrho = \frac{\rho}{-\eta^{(t)}\E_{\lambda\sim p_0}\big({L}_{\bX}(\lambda,\btheta) \big)},
\end{align*}
where $\rho=\frac{\Var(\exp(-\eta^{(t)}{L}_{\bX}(\lambda,\btheta)))}{2\big(\mathbb{E}_{\lambda\sim p_0}\exp(-\eta^{(t)}{L}_{\bX}(\lambda,\btheta))\big)^2}$.

We plot the $\varrho$ values for the heterogeneous graph datasets on EM-GCN[PT] over multiple $t$ iterations in \cref{fig:ratio}. 
We found that $\varrho$ consistently exhibits small absolute values, supporting the postulation that $-\eta^{(t)}\E_{\lambda\sim p_0}\big({L}_{\bX}(\lambda,\btheta) \big)$ is the main component in \cref{eq:lca}.

\end{document}


%% file: preamble.tex
\usepackage[T1]{fontenc}
\usepackage[utf8]{inputenc}
\usepackage{mathtools}

\usepackage{amssymb,mathrsfs}
\usepackage{amsthm}
\usepackage{bm}
\usepackage{scalerel}
\usepackage{nicefrac}
\usepackage{microtype} 
\usepackage[shortlabels]{enumitem}
\usepackage{graphicx}
\usepackage{epstopdf}
\DeclareGraphicsExtensions{.eps,.png,.jpg,.pdf}

\usepackage{url}
\usepackage{colortbl}
\usepackage{booktabs}
\usepackage{multirow}
\usepackage{colortbl,xcolor}
\usepackage[normalem]{ulem}
\usepackage{xparse,xstring}
\usepackage{calc}
\usepackage{etoolbox}

\usepackage{array}
\newcolumntype{L}[1]{>{\raggedright\let\newline\\\arraybackslash\hspace{0pt}}m{#1}}
\newcolumntype{C}[1]{>{\centering\let\newline\\\arraybackslash\hspace{0pt}}m{#1}}
\newcolumntype{R}[1]{>{\raggedleft\let\newline\\\arraybackslash\hspace{0pt}}m{#1}}

\makeatletter
\let\MYcaption\@makecaption
\makeatother
\usepackage[font=footnotesize]{subcaption}
\makeatletter
\let\@makecaption\MYcaption
\makeatother

\usepackage{glossaries}
\makeatletter
\sfcode`\.1006

\let\oldgls\gls
\let\oldglspl\glspl

\newcommand\fussy@ifnextchar[3]{%
	\let\reserved@d=#1%
	\def\reserved@a{#2}%
	\def\reserved@b{#3}%
	\futurelet\@let@token\fussy@ifnch}
\def\fussy@ifnch{%
	\ifx\@let@token\reserved@d
		\let\reserved@c\reserved@a
	\else
		\let\reserved@c\reserved@b
	\fi
	\reserved@c}

\renewcommand{\gls}[1]{%
\oldgls{#1}\fussy@ifnextchar.{\@checkperiod}{\@}}
\renewcommand{\glspl}[1]{%
\oldglspl{#1}\fussy@ifnextchar.{\@checkperiod}{\@}}

\newcommand{\@checkperiod}[1]{%
	\ifnum\sfcode`\.=\spacefactor\else#1\fi
}

\robustify{\gls}
\robustify{\glspl}
\makeatother

\newacronym{wrt}{w.r.t.}{with respect to}
\newacronym{RHS}{R.H.S.}{right-hand side}
\newacronym{LHS}{L.H.S.}{left-hand side}
\newacronym{iid}{i.i.d.}{independent and identically distributed}
\newacronym{SOTA}{SOTA}{state-of-the-art}

\usepackage{float}

\usepackage[capitalize, noabbrev]{cleveref}
\crefname{equation}{}{}
\Crefname{equation}{}{}
\crefname{claim}{claim}{claims}
\crefname{step}{step}{steps}
\crefname{line}{line}{lines}
\crefname{condition}{condition}{conditions}
\crefname{dmath}{}{}
\crefname{dseries}{}{}
\crefname{dgroup}{}{}

\crefname{Problem}{Problem}{Problems}
\crefformat{Problem}{Problem~#2#1#3}
\crefrangeformat{Problem}{Problems~#3#1#4 to~#5#2#6}

\crefname{Theorem}{Theorem}{Theorems}
\crefname{Corollary}{Corollary}{Corollaries}
\crefname{Proposition}{Proposition}{Propositions}
\crefname{Lemma}{Lemma}{Lemmas}
\crefname{Definition}{Definition}{Definitions}
\crefname{Example}{Example}{Examples}
\crefname{Assumption}{Assumption}{Assumptions}
\crefname{Remark}{Remark}{Remarks}
\crefname{Rem}{Remark}{Remarks}
\crefname{remarks}{Remarks}{Remarks}
\crefname{Appendix}{Appendix}{Appendices}
\crefname{Supplement}{Supplement}{Supplements}
\crefname{Exercise}{Exercise}{Exercises}
\crefname{Theorem_A}{Theorem}{Theorems}
\crefname{Corollary_A}{Corollary}{Corollaries}
\crefname{Proposition_A}{Proposition}{Propositions}
\crefname{Lemma_A}{Lemma}{Lemmas}
\crefname{Definition_A}{Definition}{Definitions}

\usepackage{crossreftools}
\ifx\notloadhyperref\undefined
\else
	\relax
\fi

\ifx\loadbreqn\undefined
	\relax
\else
	\usepackage{breqn}
\fi

\makeatletter
\def\cleartheorem#1{%
    \expandafter\let\csname#1\endcsname\relax
    \expandafter\let\csname c@#1\endcsname\relax
}
\def\clearthms#1{ \@for\tname:=#1\do{\cleartheorem\tname} }
\makeatother

\ifx\renewtheorem\undefined
	\ifx\useTheoremCounter\undefined
		\newtheorem{Theorem}{Theorem}
		\newtheorem{Corollary}{Corollary}
		\newtheorem{Proposition}{Proposition}
		
	\else
		\newtheorem{Theorem}{Theorem}

	\fi

	\newtheorem{Definition}{Definition}
	\newtheorem{Example}{Example}
	\newtheorem{Remark}{Remark}

\fi

\theoremstyle{remark}

\theoremstyle{plain}

\newcommand{\qednew}{\nobreak \ifvmode \relax \else
		\ifdim\lastskip<1.5em \hskip-\lastskip
			\hskip1.5em plus0em minus0.5em \fi \nobreak
		\vrule height0.75em width0.5em depth0.25em\fi}



\NewDocumentCommand{\movedownsub}{e{^_}}{%
	\IfNoValueTF{#1}{%
		\IfNoValueF{#2}{^{}}
	}{%
		^{#1}
	}%
	\IfNoValueF{#2}{_{#2}}
}

\let\latexchi\chi
\RenewDocumentCommand{\chi}{}{\latexchi\movedownsub}

\newcommand{\calP}{\mathcal{P}}

\newcommand{\bx}{\mathbf{x}}
\newcommand{\bX}{\mathbf{X}}
\newcommand{\by}{\mathbf{y}}

\newcommand{\bz}{\mathbf{z}}

\DeclareSymbolFont{bsfletters}{OT1}{cmss}{bx}{n}
\DeclareSymbolFont{ssfletters}{OT1}{cmss}{m}{n}
\DeclareMathSymbol{\bsfGamma}{0}{bsfletters}{'000}
\DeclareMathSymbol{\ssfGamma}{0}{ssfletters}{'000}
\DeclareMathSymbol{\bsfDelta}{0}{bsfletters}{'001}
\DeclareMathSymbol{\ssfDelta}{0}{ssfletters}{'001}
\DeclareMathSymbol{\bsfTheta}{0}{bsfletters}{'002}
\DeclareMathSymbol{\ssfTheta}{0}{ssfletters}{'002}
\DeclareMathSymbol{\bsfLambda}{0}{bsfletters}{'003}
\DeclareMathSymbol{\ssfLambda}{0}{ssfletters}{'003}
\DeclareMathSymbol{\bsfXi}{0}{bsfletters}{'004}
\DeclareMathSymbol{\ssfXi}{0}{ssfletters}{'004}
\DeclareMathSymbol{\bsfPi}{0}{bsfletters}{'005}
\DeclareMathSymbol{\ssfPi}{0}{ssfletters}{'005}
\DeclareMathSymbol{\bsfSigma}{0}{bsfletters}{'006}
\DeclareMathSymbol{\ssfSigma}{0}{ssfletters}{'006}
\DeclareMathSymbol{\bsfUpsilon}{0}{bsfletters}{'007}
\DeclareMathSymbol{\ssfUpsilon}{0}{ssfletters}{'007}
\DeclareMathSymbol{\bsfPhi}{0}{bsfletters}{'010}
\DeclareMathSymbol{\ssfPhi}{0}{ssfletters}{'010}
\DeclareMathSymbol{\bsfPsi}{0}{bsfletters}{'011}
\DeclareMathSymbol{\ssfPsi}{0}{ssfletters}{'011}
\DeclareMathSymbol{\bsfOmega}{0}{bsfletters}{'012}
\DeclareMathSymbol{\ssfOmega}{0}{ssfletters}{'012}

\newcommand{\btheta}{\bm{\theta}}

\newcommand{\blambda}{\bm{\lambda}}

\newcommand{\bLambda}{\bm{\Lambda}}

\makeatletter
\newcommand*\rel@kern[1]{\kern#1\dimexpr\macc@kerna}
\newcommand*\widebar[1]{%
  \begingroup
  \def\mathaccent##1##2{%
    \rel@kern{0.8}%
    \overline{\rel@kern{-0.8}\macc@nucleus\rel@kern{0.2}}%
    \rel@kern{-0.2}%
  }%
  \macc@depth\@ne
  \let\math@bgroup\@empty \let\math@egroup\macc@set@skewchar
  \mathsurround\z@ \frozen@everymath{\mathgroup\macc@group\relax}%
  \macc@set@skewchar\relax
  \let\mathaccentV\macc@nested@a
  \macc@nested@a\relax111{#1}%
  \endgroup
}
\makeatother

\DeclareMathOperator*{\argmax}{arg\,max}

\DeclareMathOperator{\var}{var}

\DeclareMathOperator{\cov}{cov}

\newcommand{\ifbcdot}[1]{\ifblank{#1}{\cdot}{#1}}

\DeclarePairedDelimiterX\abs[1]{\lvert}{\rvert}{\ifbcdot{#1}}
\DeclarePairedDelimiterX\parens[1]{(}{)}{\ifbcdot{#1}}
\DeclarePairedDelimiterX\brk[1]{[}{]}{\ifbcdot{#1}}
\DeclarePairedDelimiterX\braces[1]{\{}{\}}{\ifbcdot{#1}}
\DeclarePairedDelimiterX\angles[1]{\langle}{\rangle}{\ifblank{#1}{\cdot,\cdot}{#1}}
\DeclarePairedDelimiterX\ip[2]{\langle}{\rangle}{\ifbcdot{#1},\ifbcdot{#2}}
\DeclarePairedDelimiterX\norm[1]{\lVert}{\rVert}{\ifbcdot{#1}}
\DeclarePairedDelimiterX\ceil[1]{\lceil}{\rceil}{\ifbcdot{#1}}
\DeclarePairedDelimiterX\floor[1]{\lfloor}{\rfloor}{\ifbcdot{#1}}

\DeclareFontFamily{U}{matha}{\hyphenchar\font45}
\DeclareFontShape{U}{matha}{m}{n}{
      <5> <6> <7> <8> <9> <10> gen * matha
      <10.95> matha10 <12> <14.4> <17.28> <20.74> <24.88> matha12
      }{}
\DeclareSymbolFont{matha}{U}{matha}{m}{n}
\DeclareFontSubstitution{U}{matha}{m}{n}

\DeclareFontFamily{U}{mathx}{\hyphenchar\font45}
\DeclareFontShape{U}{mathx}{m}{n}{
      <5> <6> <7> <8> <9> <10>
      <10.95> <12> <14.4> <17.28> <20.74> <24.88>
      mathx10
      }{}
\DeclareSymbolFont{mathx}{U}{mathx}{m}{n}
\DeclareFontSubstitution{U}{mathx}{m}{n}

\DeclareMathDelimiter{\vvvert}{0}{matha}{"7E}{mathx}{"17}
\DeclarePairedDelimiterX\vertiii[1]{\vvvert}{\vvvert}{\ifbcdot{#1}}

\DeclarePairedDelimiterXPP\trace[1]{\operatorname{Tr}}{(}{)}{}{\ifbcdot{#1}} 
\DeclarePairedDelimiterXPP\col[1]{\operatorname{col}}{\{}{\}}{}{\ifbcdot{#1}} 
\DeclarePairedDelimiterXPP\row[1]{\operatorname{row}}{\{}{\}}{}{\ifbcdot{#1}} 
\DeclarePairedDelimiterXPP\erf[1]{\operatorname{erf}}{(}{)}{}{\ifbcdot{#1}}
\DeclarePairedDelimiterXPP\erfc[1]{\operatorname{erfc}}{(}{)}{}{\ifbcdot{#1}}
\DeclarePairedDelimiterXPP\KLD[2]{D}{(}{)}{}{\ifbcdot{#1}\, \delimsize\|\, \ifbcdot{#2}} 
\DeclarePairedDelimiterXPP\op[2]{\operatorname{#1}}{(}{)}{}{#2} 

\newcommand{\ud}{\,\mathrm{d}} 

\DeclarePairedDelimiterXPP\indicate[1]{{\bf 1}}{\{}{\}}{}{\ifbcdot{#1}}
\newcommand{\indicator}[1]{{\bf 1}_{\braces*{\ifbcdot{#1}}}}

\newcommand{\tc}[1]{^{(#1)}}
\NewDocumentCommand\ofrac{s m}{%
	\IfBooleanTF#1%
	{\dfrac{1}{#2}}%
	{\frac{1}{#2}}%
}
\NewDocumentCommand\ddfrac{s m m}{%
	\IfBooleanTF#1%
	{\dfrac{\mathrm{d} {#2}}{\mathrm{d} {#3}}}%
	{\frac{\mathrm{d} {#2}}{\mathrm{d} {#3}}}%
}
\NewDocumentCommand\ppfrac{s m m}{%
	\IfBooleanTF#1%
	{\dfrac{\partial {#2}}{\partial {#3}}}%
	{\frac{\partial {#2}}{\partial {#3}}}%
}

\providecommand\given{}

\DeclarePairedDelimiterX\Set[2]\{\}{%
\renewcommand\given{\SetSymbol[\delimsize]{#1}}
#2
}
\DeclarePairedDelimiterX\Setc[1]\{\}{%
\renewcommand\given{\SetSymbol{:}}
#1
}

\NewDocumentCommand\set{s o m}{%
	\IfBooleanTF#1%
	{\IfValueTF{#2}{\Set*{#2}{#3}}{\Setc*{#3}}}%
	{\IfValueTF{#2}{\Set{#2}{#3}}{\Setc{#3}}}%
}

\NewDocumentCommand{\evalat}{ s O{\big} m e{_^} }{%
\IfBooleanTF{#1}%
{\left. #3 \right|}{#3#2|}%
\IfValueT{#4}{_{#4}}%
\IfValueT{#5}{^{#5}}%
}

\providecommand\given{}
\DeclarePairedDelimiterXPP\cprob[1]{}(){}{
\renewcommand\given{\nonscript\,\delimsize\vert\allowbreak\nonscript\,\mathopen{}}%
#1%
}
\DeclarePairedDelimiterXPP\cexp[1]{}[]{}{
\renewcommand\given{\nonscript\,\delimsize\vert\allowbreak\nonscript\,\mathopen{}}%
#1%
}

\DeclareDocumentCommand \P { s e{_^} d() g } {%
	\mathbb{P}%
	\IfBooleanTF{#1}%
		{
			\IfValueT{#2}{_{#2}}%
			\IfValueT{#3}{^{#3}}%
			\IfValueTF{#5}{\cprob{#4 \given #5}}{\IfValueT{#4}{\cprob{#4}}}%
		}%
		{
			\IfValueT{#2}{_{#2}}%
			\IfValueT{#3}{^{#3}}%
			\IfValueTF{#5}{\cprob*{#4 \given #5}}{\IfValueT{#4}{\cprob*{#4}}}%
		}%
}

\DeclareDocumentCommand \E { s e{_^} o g } {%
	\mathbb{E}%
	\IfBooleanTF{#1}%
		{
			\IfValueT{#2}{_{#2}}%
			\IfValueT{#3}{^{#3}}%
			\IfValueTF{#5}{\cexp{#4 \given #5}}{\IfValueT{#4}{\cexp{#4}}}%
		}%
		{
			\IfValueT{#2}{_{#2}}%
			\IfValueT{#3}{^{#3}}%
			\IfValueTF{#5}{\cexp*{#4 \given #5}}{\IfValueT{#4}{\cexp*{#4}}}%
		}%
}

\DeclareDocumentCommand \Var { s e{_^} d() g } {%
	\var%
	\IfBooleanTF{#1}%
		{
			\IfValueT{#2}{_{#2}}%
			\IfValueT{#3}{^{#3}}%
			\IfValueTF{#5}{\cprob{#4 \given #5}}{\IfValueT{#4}{\cprob{#4}}}%
		}%
		{
			\IfValueT{#2}{_{#2}}%
			\IfValueT{#3}{^{#3}}%
			\IfValueTF{#5}{\cprob*{#4 \given #5}}{\IfValueT{#4}{\cprob*{#4}}}%
		}%
}

\DeclareDocumentCommand \Cov { s e{_^} d() g } {%
	\cov%
	\IfBooleanTF{#1}%
		{
			\IfValueT{#2}{_{#2}}%
			\IfValueT{#3}{^{#3}}%
			\IfValueTF{#5}{\cprob{#4 \given #5}}{\IfValueT{#4}{\cprob{#4}}}%
		}%
		{
			\IfValueT{#2}{_{#2}}%
			\IfValueT{#3}{^{#3}}%
			\IfValueTF{#5}{\cprob*{#4 \given #5}}{\IfValueT{#4}{\cprob*{#4}}}%
		}%
}

\ExplSyntaxOn
\NewDocumentCommand \dist {m o o} {%
\mathrm{#1}\left(%
	\IfValueT{#3}{%
		\tl_if_blank:nTF{ #3 }{\cdot\, \middle|\, }{#3\, \middle|\, }%
	}
	\IfValueT{#2}{#2}%
\right)%
}
\ExplSyntaxOff

\NewDocumentCommand {\cbrace} {t+ D[]{black} D(){\widthof{#5}} m m } {%
	\begingroup%
		\color{#2}
		\IfBooleanTF{#1}{%
			\overbrace{#4}^%
		}{
			\underbrace{#4}_%
		}%
		{\parbox[c]{#3}{\centering\footnotesize{#5}}}%
	\endgroup%
}

\let\oldforall\forall
\renewcommand{\forall}{\oldforall \, }

\let\oldexist\exists
\renewcommand{\exists}{\oldexist \, }

\makeatletter
\newcommand{\udcloser}[1]{\underline{\smash{#1}}}

\newcommand{\rankcolor}[2]{%
	\expandafter\renewcommand\csname #1\endcsname[1]{%
		\ifblank{##1}{%
			{\color{#2} \textbf{#2}}%
		}{%
			\ifmmode
				\textcolor{#2}{\bm{##1}}%
			\else%
				{\color{#2} \textbf{##1}}%
			\fi	
		}%
	}
}

\rankcolor{first}{red}
\rankcolor{second}{blue}
\rankcolor{third}{cyan}
\makeatother

\newcommand{\figref}[1]{Fig.~\ref{#1}}
\graphicspath{{./Figures/}{./figures/}}
\DeclareDocumentCommand{\includeCroppedPdf}{ o O{./Figures/} m }{
	\IfFileExists{#2#3-crop.pdf}{}{%
		\immediate\write18{pdfcrop #2#3.pdf #2#3-crop.pdf}}%
	\includegraphics[#1]{#2#3-crop.pdf}
}

\makeatletter
\newcommand*{\addFileDependency}[1]{
  \typeout{(#1)}
  \@addtofilelist{#1}
  \IfFileExists{#1}{}{\typeout{No file #1.}}
}
\makeatother

\definecolor{gray90}{gray}{0.9}
\def\colorlist{red,blue,brown,cyan,darkgray,gray,lightgray,green,lime,magenta,olive,orange,pink,purple,teal,violet,white,yellow}

\makeatletter
\def\startcomment{[}
\ifx\nohighlights\undefined
	\newcommand{\createcolor}[1]{%
			\expandafter\newcommand\csname #1\endcsname[1]{{\color{#1} ##1}}%
	}
	\newcommand{\msout}[1]{\text{\color{green} \sout{\ensuremath{#1}}}}
	\newcommand{\del}[1]{{\color{green}\ifmmode \msout{#1}\else\sout{#1}\fi}}
\else
	\newcommand{\createcolor}[1]{%
			\expandafter\newcommand\csname #1\endcsname[1]{%
				\noexpandarg%
				\StrChar{##1}{1}[\firstletter]%
				\if\firstletter\startcomment%
					\relax
				\else%
					##1
				\fi
			}%
	}
	\newcommand{\msout}[1]{}
	\newcommand{\del}[1]{}
\fi

\def\@tempa#1,{%
    \ifx\relax#1\relax\else
        \createcolor{#1}%
        \expandafter\@tempa
    \fi
}
\expandafter\@tempa\colorlist,\relax,
\makeatother

\newcommand{\hhide}[1]{}

\ifx\diagnoselabel\undefined
	\relax
\else
	\makeatletter
	\def\@testdef #1#2#3{%
		\def\reserved@a{#3}\expandafter \ifx \csname #1@#2\endcsname
			\reserved@a  \else
			\typeout{^^Jlabel #2 changed:^^J%
				\meaning\reserved@a^^J%
				\expandafter\meaning\csname #1@#2\endcsname^^J}%
			\@tempswatrue \fi}
	\makeatother
\fi


%% file: main_icml.bbl
\begin{thebibliography}{48}
\providecommand{\natexlab}[1]{#1}
\providecommand{\url}[1]{\texttt{#1}}
\expandafter\ifx\csname urlstyle\endcsname\relax
  \providecommand{\doi}[1]{doi: #1}\else
  \providecommand{\doi}{doi: \begingroup \urlstyle{rm}\Url}\fi

\bibitem[Bishop(2006)]{Bis06}
Bishop, C.
\newblock \emph{Pattern Recognition and Machine Learning}.
\newblock Springer, 2006.

\bibitem[Brody et~al.(2022)Brody, Alon, and Yahav]{brody2022how}
Brody, S., Alon, U., and Yahav, E.
\newblock How attentive are graph attention networks?
\newblock In \emph{International Conference on Learning Representations}, 2022.

\bibitem[Cen et~al.(2023)Cen, Hou, Wang, Chen, Luo, Yu, Zhang, Yao, Zeng, Guo, Dong, Yang, Zhang, Dai, Wang, Zhou, Yang, and Tang]{cen2023cogdl}
Cen, Y., Hou, Z., Wang, Y., Chen, Q., Luo, Y., Yu, Z., Zhang, H., Yao, X., Zeng, A., Guo, S., Dong, Y., Yang, Y., Zhang, P., Dai, G., Wang, Y., Zhou, C., Yang, H., and Tang, J.
\newblock Cogdl: A comprehensive library for graph deep learning.
\newblock In \emph{Proceedings of the ACM Web Conference 2023 (WWW'23)}, 2023.

\bibitem[Chen et~al.(2020)Chen, Wei, Huang, Ding, and Li]{gcnv2}
Chen, M., Wei, Z., Huang, Z., Ding, B., and Li, Y.
\newblock Simple and deep graph convolutional networks.
\newblock In \emph{Proceedings of the 37th International Conference on Machine Learning}, volume 119 of \emph{Proceedings of Machine Learning Research}, pp.\  1725--1735. PMLR, 13--18 Jul 2020.

\bibitem[Chen et~al.(2022)Chen, Yang, Zhang, Zhao, Meng, Hao, and King]{recommender2022}
Chen, Y., Yang, M., Zhang, Y., Zhao, M., Meng, Z., Hao, J., and King, I.
\newblock Modeling scale-free graphs with hyperbolic geometry for knowledge-aware recommendation.
\newblock In \emph{Proceedings of the 15th ACM International Conference on Web Search and Data Mining}, WSDM '22, pp.\  94–102, New York, NY, USA, 2022. Association for Computing Machinery.
\newblock ISBN 9781450391320.
\newblock \doi{10.1145/3488560.3498419}.
\newblock URL \url{https://doi.org/10.1145/3488560.3498419}.

\bibitem[Dai et~al.(2022)Dai, Jin, Liu, and Wang]{Dai22}
Dai, E., Jin, W., Liu, H., and Wang, S.
\newblock Towards robust graph neural networks for noisy graphs with sparse labels.
\newblock In \emph{Proc.\ of the 15th ACM International WSDM Conference}, 2022.

\bibitem[Dong \& Kluger(2023)Dong and Kluger]{gsnoise2023}
Dong, M. and Kluger, Y.
\newblock Towards understanding and reducing graph structural noise for gnns.
\newblock In \emph{Proceedings of the 40th International Conference on Machine Learning}, ICML'23. JMLR.org, 2023.

\bibitem[Gilmer et~al.(2017{\natexlab{a}})Gilmer, Schoenholz, Riley, Vinyals, and Dahl]{chemistry2017}
Gilmer, J., Schoenholz, S.~S., Riley, P.~F., Vinyals, O., and Dahl, G.~E.
\newblock Neural message passing for quantum chemistry.
\newblock In \emph{Proceedings of the 34th International Conference on Machine Learning}, ICML'17, pp.\  1263–1272. JMLR.org, 2017{\natexlab{a}}.

\bibitem[Gilmer et~al.(2017{\natexlab{b}})Gilmer, Schoenholz, Riley, Vinyals, and Dahl]{mpnn2017}
Gilmer, J., Schoenholz, S.~S., Riley, P.~F., Vinyals, O., and Dahl, G.~E.
\newblock Neural message passing for quantum chemistry.
\newblock In \emph{Proceedings of the 34th International Conference on Machine Learning}, volume~70 of \emph{Proceedings of Machine Learning Research}, pp.\  1263--1272. PMLR, 06--11 Aug 2017{\natexlab{b}}.

\bibitem[Guedj(2019)]{Gue19}
Guedj, B.
\newblock A primer on {PAC-Bayesian} learning.
\newblock \emph{arXiv preprint arXiv:1901.05353}, 2019.

\bibitem[Hardt et~al.(2016)Hardt, Recht, and Singer]{Har16}
Hardt, M., Recht, B., and Singer, Y.
\newblock Train faster, generalize better: stability of stochastic gradient descent.
\newblock In \emph{Proceedings of the 33th International Conference on Machine Learning}, volume~48 of \emph{Proceedings of Machine Learning Research}, pp.\  1225--1234. PMLR, Jun 2016.

\bibitem[Ji et~al.(2023{\natexlab{a}})Ji, Jian, and Tay]{Jif23}
Ji, F., Jian, X., and Tay, W.~P.
\newblock On distributional graph signals.
\newblock \emph{arXiv:2302.11104}, 2023{\natexlab{a}}.

\bibitem[Ji et~al.(2023{\natexlab{b}})Ji, Lee, Zhao, Tay, and Yang]{Ji23}
Ji, F., Lee, S., Zhao, K., Tay, W.~P., and Yang, J.
\newblock Distributional signals for node classification in graph neural networks.
\newblock \emph{arXiv:2304.03507}, 2023{\natexlab{b}}.

\bibitem[Ji et~al.(2023{\natexlab{c}})Ji, Lee, Meng, Zhao, Yang, and Tay]{jifeng23}
Ji, F., Lee, S.~H., Meng, H., Zhao, K., Yang, J., and Tay, W.~P.
\newblock Leveraging label non-uniformity for node classification in graph neural networks.
\newblock In \emph{International Conference on Machine Learning}, 2023{\natexlab{c}}.

\bibitem[Ji et~al.(2023{\natexlab{d}})Ji, Tay, and Ortega]{gsptsp2023}
Ji, F., Tay, W.~P., and Ortega, A.
\newblock Graph signal processing over a probability space of shift operators.
\newblock \emph{IEEE Transactions on Signal Processing}, 71:\penalty0 1159--1174, 2023{\natexlab{d}}.
\newblock \doi{10.1109/TSP.2023.3263675}.

\bibitem[Kang et~al.(2023)Kang, Zhao, Song, Wang, and Tay]{KanZhaSon}
Kang, Q., Zhao, K., Song, Y., Wang, S., and Tay, W.~P.
\newblock Node embedding from neural {Hamiltonian} orbits in graph neural networks.
\newblock In \emph{Proc. International Conference on Machine Learning}, Haiwaii, USA, Jul. 2023.

\bibitem[Kearnes et~al.(2016)Kearnes, McCloskey, Berndl, Pande, and Riley]{Kearnes_2016}
Kearnes, S., McCloskey, K., Berndl, M., Pande, V., and Riley, P.
\newblock Molecular graph convolutions: moving beyond fingerprints.
\newblock \emph{Journal of Computer-Aided Molecular Design}, pp.\  595--608, 2016.
\newblock \doi{10.1007/s10822-016-9938-8}.

\bibitem[Kipf \& Welling(2017)Kipf and Welling]{kipf2016semi}
Kipf, T.~N. and Welling, M.
\newblock Semi-supervised classification with graph convolutional networks.
\newblock \emph{International Conference on Learning Representations (ICLR)}, 2017.

\bibitem[Lee et~al.(2021)Lee, Ji, and Tay]{9413417}
Lee, S.~H., Ji, F., and Tay, W.~P.
\newblock Learning on heterogeneous graphs using high-order relations.
\newblock In \emph{ICASSP 2021 - 2021 IEEE International Conference on Acoustics, Speech and Signal Processing (ICASSP)}, pp.\  3175--3179, 2021.
\newblock \doi{10.1109/ICASSP39728.2021.9413417}.

\bibitem[Lee et~al.(2022)Lee, Ji, and Tay]{Lee2022SGATSG}
Lee, S.~H., Ji, F., and Tay, W.~P.
\newblock {SGAT}: Simplicial graph attention network.
\newblock In \emph{International Joint Conference on Artificial Intelligence}, 2022.

\bibitem[Li et~al.(2021{\natexlab{a}})Li, Cao, Zhu, Liu, Zhu, and Wu]{li2021cgnn}
Li, H., Cao, J., Zhu, J., Liu, Y., Zhu, Q., and Wu, G.
\newblock Curvature graph neural network.
\newblock \emph{Information Sciences}, 2021{\natexlab{a}}.
\newblock \doi{10.1016/j.ins.2021.12.077}.

\bibitem[Li et~al.(2021{\natexlab{b}})Li, Zhou, Hu, Fan, Zhang, Gu, and Karypis]{dgllife}
Li, M., Zhou, J., Hu, J., Fan, W., Zhang, Y., Gu, Y., and Karypis, G.
\newblock Dgl-lifesci: An open-source toolkit for deep learning on graphs in life science.
\newblock \emph{ACS Omega}, 2021{\natexlab{b}}.

\bibitem[Li et~al.(2021{\natexlab{c}})Li, Yuan, Radfar, Marendy, Ni, O’Brien, and Casillas-Espinosa]{gspbioreview}
Li, R., Yuan, X., Radfar, M., Marendy, P., Ni, W., O’Brien, T., and Casillas-Espinosa, P.
\newblock Graph signal processing, graph neural network and graph learning on biological data: A systematic review.
\newblock \emph{IEEE Reviews in Biomedical Engineering}, PP:\penalty0 1--1, 10 2021{\natexlab{c}}.
\newblock \doi{10.1109/RBME.2021.3122522}.

\bibitem[Liu et~al.(2022)Liu, Zhang, Wang, He, Caverlee, Chan, Yeung, and Heng]{ecommerce}
Liu, W., Zhang, Y., Wang, J., He, Y., Caverlee, J., Chan, P. P.~K., Yeung, D.~S., and Heng, P.-A.
\newblock Item relationship graph neural networks for e-commerce.
\newblock \emph{IEEE Transactions on Neural Networks and Learning Systems}, 33\penalty0 (9):\penalty0 4785--4799, 2022.
\newblock \doi{10.1109/TNNLS.2021.3060872}.

\bibitem[Luan et~al.(2022)Luan, Hua, Lu, Zhu, Zhao, Zhang, Chang, and Precup]{luan2022revisiting}
Luan, S., Hua, C., Lu, Q., Zhu, J., Zhao, M., Zhang, S., Chang, X.-W., and Precup, D.
\newblock Revisiting heterophily for graph neural networks.
\newblock \emph{Advances in neural information processing systems}, 35:\penalty0 1362--1375, 2022.

\bibitem[Lv et~al.(2021)Lv, Ding, Liu, Chen, Feng, He, Zhou, Jiang, Dong, and Tang]{hgb}
Lv, Q., Ding, M., Liu, Q., Chen, Y., Feng, W., He, S., Zhou, C., Jiang, J., Dong, Y., and Tang, J.
\newblock Are we really making much progress? revisiting, benchmarking and refining heterogeneous graph neural networks.
\newblock In \emph{Proceedings of the 27th ACM SIGKDD Conference on Knowledge Discovery \& Data Mining}, KDD '21, pp.\  1150–1160, 2021.
\newblock \doi{10.1145/3447548.3467350}.

\bibitem[Michael~Defferrard(2016)]{Def16}
Michael~Defferrard, Xavier~Bresson, P.~V.
\newblock Convolutional neural networks on graphs with fast localized spectral filtering.
\newblock In \emph{Proc.\ of the 29th International Conference on Neural Information Processing Systems}, 2016.

\bibitem[Pei et~al.(2020)Pei, Wei, Chang, Lei, and Yang]{pei2020geomgcn}
Pei, H., Wei, B., Chang, K. C.-C., Lei, Y., and Yang, B.
\newblock Geom-gcn: Geometric graph convolutional networks, 2020.

\bibitem[Rong et~al.(2020)Rong, Huang, Xu, and Huang]{Ron20}
Rong, Y., Huang, W., Xu, T., and Huang, J.
\newblock {DropEdge}: Towards deep graph convolutional networks on node classification.
\newblock In \emph{International Conference on Learning Representations}, 2020.

\bibitem[Sawhney et~al.(2021)Sawhney, Agarwal, Wadhwa, and Shah]{financialhyp2021}
Sawhney, R., Agarwal, S., Wadhwa, A., and Shah, R.
\newblock Exploring the scale-free nature of stock markets: Hyperbolic graph learning for algorithmic trading.
\newblock In \emph{Proceedings of the Web Conference 2021}, WWW '21, pp.\  11–22, New York, NY, USA, 2021. Association for Computing Machinery.
\newblock ISBN 9781450383127.
\newblock \doi{10.1145/3442381.3450095}.
\newblock URL \url{https://doi.org/10.1145/3442381.3450095}.

\bibitem[Shen et~al.(2023)Shen, Luo, and Xia]{shen2023molecular}
Shen, C., Luo, J., and Xia, K.
\newblock Molecular geometric deep learning, 2023.

\bibitem[Shui \& Karypis(2020)Shui and Karypis]{hmg2020}
Shui, Z. and Karypis, G.
\newblock Heterogeneous molecular graph neural networks for predicting molecule properties.
\newblock In \emph{2020 IEEE International Conference on Data Mining (ICDM)}, pp.\  492--500, Los Alamitos, CA, USA, nov 2020. IEEE Computer Society.
\newblock \doi{10.1109/ICDM50108.2020.00058}.
\newblock URL \url{https://doi.ieeecomputersociety.org/10.1109/ICDM50108.2020.00058}.

\bibitem[Shuman et~al.(2013)Shuman, Narang, Frossard, Ortega, and Vandergheynst]{Shu13}
Shuman, D.~I., Narang, S.~K., Frossard, P., Ortega, A., and Vandergheynst, P.
\newblock The emerging field of signal processing on graphs: Extending high-dimensional data analysis to networks and other irregular domains.
\newblock \emph{IEEE Signal Processing Magazine}, 30\penalty0 (3):\penalty0 83--98, 2013.

\bibitem[Teh et~al.(2006)Teh, Newman, and Welling]{Teh06}
Teh, Y.~W., Newman, D., and Welling, M.
\newblock A collapsed variational bayesian inference algorithm for latent dirichlet allocation.
\newblock In \emph{Proc.\ of the 19th International Conference on Neural Information Processing Systems}, 2006.

\bibitem[Thomas et~al.(2018)Thomas, Smidt, Kearnes, Yang, Li, Kohlhoff, and Riley]{thomas2018tensor}
Thomas, N., Smidt, T., Kearnes, S., Yang, L., Li, L., Kohlhoff, K., and Riley, P.
\newblock Tensor field networks: Rotation- and translation-equivariant neural networks for 3d point clouds, 2018.

\bibitem[Topping et~al.(2022)Topping, Giovanni, Chamberlain, Dong, and Bronstein]{oversquashing2022}
Topping, J., Giovanni, F.~D., Chamberlain, B.~P., Dong, X., and Bronstein, M.~M.
\newblock Understanding over-squashing and bottlenecks on graphs via curvature.
\newblock \emph{International Conference on Learning Representations}, 2022.

\bibitem[Veli{\v{c}}kovi{\'{c}} et~al.(2018)Veli{\v{c}}kovi{\'{c}}, Cucurull, Casanova, Romero, Li{\`{o}}, and Bengio]{velickovic2018graph}
Veli{\v{c}}kovi{\'{c}}, P., Cucurull, G., Casanova, A., Romero, A., Li{\`{o}}, P., and Bengio, Y.
\newblock {Graph Attention Networks}.
\newblock \emph{International Conference on Learning Representations}, 2018.

\bibitem[Villani(2009)]{Vil09}
Villani, C.
\newblock \emph{Optimal Transport, Old and New}.
\newblock Springer, 2009.

\bibitem[Wang et~al.(2019)Wang, Zheng, Ye, Gan, Li, Song, Zhou, Ma, Yu, Gai, Xiao, He, Karypis, Li, and Zhang]{wang2019dgl}
Wang, M., Zheng, D., Ye, Z., Gan, Q., Li, M., Song, X., Zhou, J., Ma, C., Yu, L., Gai, Y., Xiao, T., He, T., Karypis, G., Li, J., and Zhang, Z.
\newblock Deep graph library: A graph-centric, highly-performant package for graph neural networks.
\newblock \emph{arXiv preprint arXiv:1909.01315}, 2019.

\bibitem[Wang et~al.(2023)Wang, Yi, Liu, Wang, and Jin]{Wan23}
Wang, Y., Yi, K., Liu, X., Wang, Y.~G., and Jin, S.
\newblock {ACMP}: {Allen-Cahn} message passing with attractive and repulsive forces for graph neural networks.
\newblock In \emph{International Conference on Learning Representations}, 2023.
\newblock URL \url{https://openreview.net/forum?id=4fZc_79Lrqs}.

\bibitem[Xiong et~al.(2020)Xiong, Wang, Liu, Zhong, Wan, Li, Li, Luo, Chen, Jiang, and Zheng]{attentivefp}
Xiong, Z., Wang, D., Liu, X., Zhong, F., Wan, X., Li, X., Li, Z., Luo, X., Chen, K., Jiang, H., and Zheng, M.
\newblock Pushing the boundaries of molecular representation for drug discovery with the graph attention mechanism.
\newblock \emph{Journal of Medicinal Chemistry}, pp.\  8749--8760, 2020.
\newblock \doi{10.1021/acs.jmedchem.9b00959}.

\bibitem[Xu et~al.(2019)Xu, Hu, Leskovec, and Jegelka]{xu2018gin}
Xu, K., Hu, W., Leskovec, J., and Jegelka, S.
\newblock How powerful are graph neural networks?
\newblock In \emph{International Conference on Learning Representations}, 2019.
\newblock URL \url{https://openreview.net/forum?id=ryGs6iA5Km}.

\bibitem[Yang et~al.(2023)Yang, Yan, Pan, Ye, and Fan]{Yang2023SimpleAE}
Yang, X., Yan, M., Pan, S., Ye, X., and Fan, D.
\newblock Simple and efficient heterogeneous graph neural network.
\newblock In \emph{AAAI Conference on Artificial Intelligence}, 2023.

\bibitem[Ye et~al.(2020)Ye, Liu, Ma, Gao, and Chen]{ze20curvature}
Ye, Z., Liu, K.~S., Ma, T., Gao, J., and Chen, C.
\newblock Curvature graph network.
\newblock \emph{Proceedings of the 8th International Conference on Learning Representations (ICLR 2020)}, April 2020.

\bibitem[Ying et~al.(2018)Ying, He, Chen, Eksombatchai, Hamilton, and Leskovec]{recommender2018}
Ying, R., He, R., Chen, K., Eksombatchai, P., Hamilton, W.~L., and Leskovec, J.
\newblock Graph convolutional neural networks for web-scale recommender systems.
\newblock In \emph{Proceedings of the 24th ACM SIGKDD International Conference on Knowledge Discovery and Data Mining}, KDD '18, pp.\  974–983, New York, NY, USA, 2018. Association for Computing Machinery.
\newblock ISBN 9781450355520.
\newblock \doi{10.1145/3219819.3219890}.
\newblock URL \url{https://doi.org/10.1145/3219819.3219890}.

\bibitem[Yun et~al.(2019)Yun, Jeong, Kim, Kang, and Kim]{yun2019gtn}
Yun, S., Jeong, M., Kim, R., Kang, J., and Kim, H.~J.
\newblock Graph transformer networks.
\newblock In \emph{Proc.\ of the 33rd International Conference on Neural Information Processing Systems}, 2019.

\bibitem[Zhang et~al.(2019)Zhang, Pal, Coates, and \"{U}stebay]{bgcn}
Zhang, Y., Pal, S., Coates, M., and \"{U}stebay, D.
\newblock Bayesian graph convolutional neural networks for semi-supervised classification.
\newblock In \emph{Proceedings of the Thirty-Third AAAI Conference on Artificial Intelligence}. AAAI Press, 2019.
\newblock ISBN 978-1-57735-809-1.
\newblock \doi{10.1609/aaai.v33i01.33015829}.
\newblock URL \url{https://doi.org/10.1609/aaai.v33i01.33015829}.

\bibitem[Zhao et~al.(2023)Zhao, Kang, Song, She, Wang, and Tay]{zhao2023graph}
Zhao, K., Kang, Q., Song, Y., She, R., Wang, S., and Tay, W.~P.
\newblock Graph neural convection-diffusion with heterophily.
\newblock In \emph{Proc. International Joint Conference on Artificial Intelligence}, Macao, China, Aug 2023.

\end{thebibliography}
